\documentclass{article}

\usepackage[preprint]{neurips_2026}

\usepackage[utf8]{inputenc}
\usepackage[T1]{fontenc}

\usepackage{hyperref}
\usepackage{url}

\usepackage{booktabs}
\usepackage{multirow}

\usepackage{amsfonts}
\usepackage{amsmath, amssymb, amsthm}
\usepackage{mathtools}
\usepackage{nicefrac}

\usepackage{microtype}

\usepackage{graphicx}
\usepackage{float}

\usepackage{algorithm}
\usepackage{algorithmic}
\providecommand{\TO}{\textbf{to}}
\providecommand{\RETURN}{\STATE \textbf{return}~}

\usepackage{enumitem}

\usepackage[table]{xcolor}

\newtheorem{theorem}{Theorem}
\newtheorem{lemma}[theorem]{Lemma}
\newtheorem{proposition}[theorem]{Proposition}

\newtheorem{remark}{Remark}
\theoremstyle{definition}
\newtheorem{definition}{Definition}
\newtheorem{assumption}{Assumption}

\newcommand{\E}{\mathbb{E}}
\newcommand{\Var}{\mathrm{Var}}
\newcommand{\Dir}{\mathrm{Dir}}
\newcommand{\Beta}{\mathrm{Beta}}
\newcommand{\neff}{n_{\mathrm{eff}}}
\newcommand{\calD}{\mathcal{D}}
\newcommand{\calC}{\mathcal{C}}

\newcommand{\bw}{h}

\title{Weighted Bayesian Conformal Prediction}

\author{%
  Xiayin Lou \\
  Chair of Cartography and Visual Analytics \\
  Technical University of Munich \\
  Munich, Germany \\
  \texttt{xiayin.lou@tum.de} \\
  \And
  Peng Luo\thanks{Corresponding author: \texttt{pengluo@mit.edu}.} \\
  Senseable City Lab \\
  Massachusetts Institute of Technology \\
  Cambridge, MA, USA \\
  \texttt{pengluo@mit.edu} \\
}

\begin{document}

\maketitle

% ============================================================
\begin{abstract}
Conformal prediction provides distribution-free prediction intervals with finite-sample coverage guarantees, and a recent line of work casts it as \emph{Bayesian Quadrature for Conformal Prediction} (BQ-CP) to deliver data-conditional guarantees through a posterior distribution over the prediction threshold. This Bayesian view, however, relies on the i.i.d.\ assumption and therefore does not extend to settings with covariate shift, spatial structure, or other forms of non-uniform calibration influence. Weighted conformal prediction addresses such settings via importance weights but remains purely frequentist, returning only a point-estimate threshold and offering no information about its reliability. We introduce \textbf{Weighted Bayesian Conformal Prediction (WBCP)}, a unified framework that brings the Bayesian threshold posterior of BQ-CP to arbitrary weighted conformal procedures. The core construction replaces the uniform Dirichlet model over calibration spacings with a weighted Dirichlet whose concentration is set by Kish's effective sample size, providing a principled bridge between the frequentist coverage of weighted conformal prediction and the meta-uncertainty quantification of Bayesian conformal prediction. We establish four theoretical results: a calibration-consistency property identifying the effective sample size as the unique variance-matching concentration; a posterior concentration rate; an extension of BQ-CP's stochastic dominance guarantee to weighted settings; and an improved conditional coverage bound. We instantiate the framework for spatial prediction as \emph{Geographical Bayesian Conformal Prediction}, where kernel-based spatial weights yield per-location threshold posteriors and interpretable spatial diagnostics, including effective sample size maps and posterior standard deviation maps. Across $13$ geospatial datasets, the proposed adaptive variant AdaGeoBCP achieves coverage comparable to the i.i.d.\ Bayesian baseline while additionally returning \emph{spatially varying} per-sample uncertainty that the i.i.d.\ baseline cannot supply. The headline contribution is therefore not improved coverage at fixed budget, but \emph{comparable coverage with per-location meta-uncertainty}---a structural feature that turns each prediction interval into a diagnostic rather than just a width.
\end{abstract}

% ============================================================
\section{Introduction}
\label{sec:intro}

\begin{figure}[t]
  \centering
  \includegraphics[width=\linewidth]{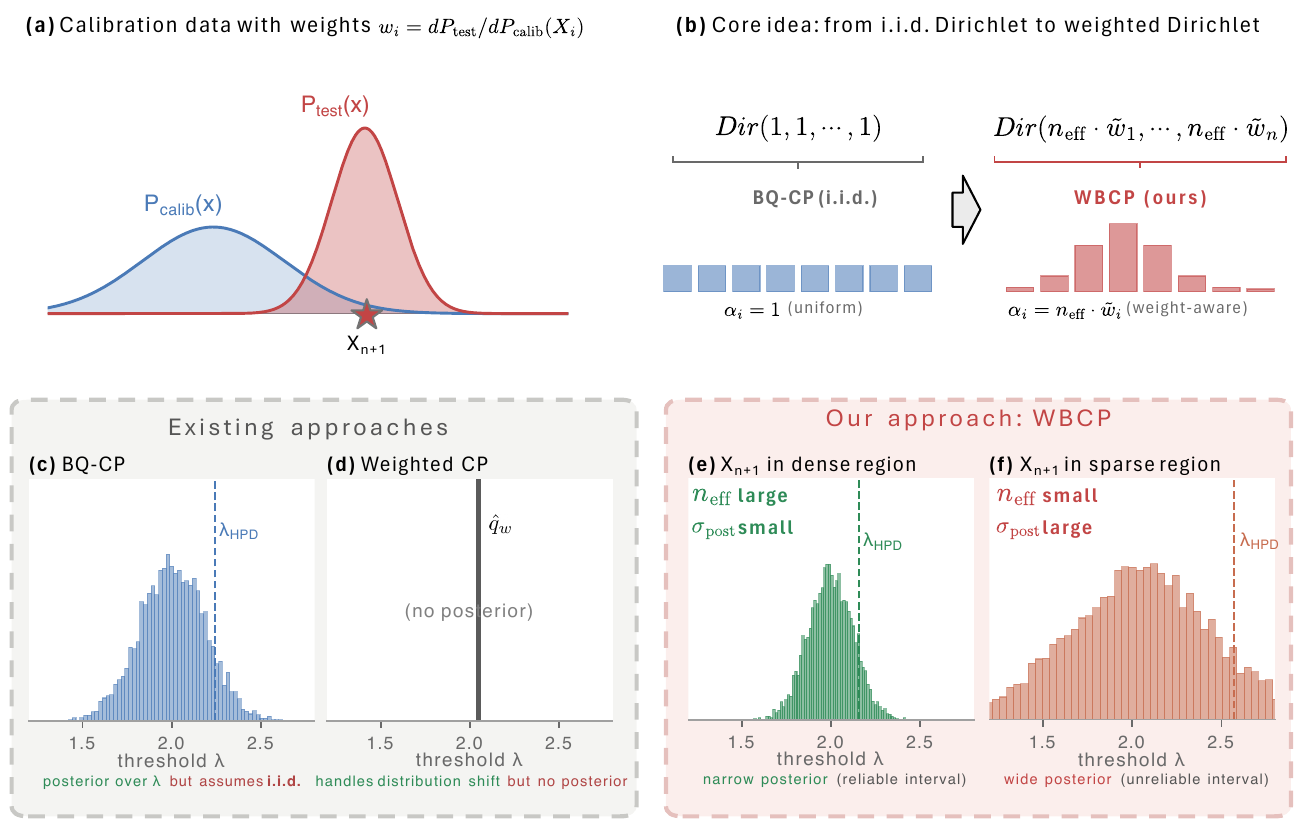}
  \caption{\textbf{WBCP overview.}
  (a) Calibration vs.\ test distribution with importance weights $w_i = dP_{\mathrm{test}}/dP_{\mathrm{calib}}(X_i)$.
  (b) WBCP replaces BQ-CP's $\Dir(1,\ldots,1)$ with the weighted Dirichlet $\Dir(\neff\,\tilde{w}_1,\ldots,\neff\,\tilde{w}_n)$.
  (c--d) Existing approaches give either a posterior \emph{or} shift-robustness, not both.
  (e--f) WBCP gives both: the threshold posterior makes interval-width reliability explicit.}
  \label{fig:wbcp_overview}
\end{figure}

Conformal prediction (CP)~\citep{vovk2005algorithmic} provides distribution-free prediction intervals with finite-sample coverage guarantees: for any model and miscoverage level~$\alpha$, the constructed prediction set contains the true outcome with probability at least $1-\alpha$. This elegant guarantee has made CP a cornerstone of modern uncertainty quantification~\citep{lei2018distribution, romano2019conformalized}.

\textbf{Two levels of uncertainty.} It is important to distinguish two fundamentally different levels of uncertainty in any prediction system. The \emph{first level} is the predictive uncertainty about the outcome~$Y$ itself: how far might the true value deviate from the prediction $\hat{f}(X)$? CP addresses this level by constructing an interval $[\hat{f}(X) \pm \hat{q}]$. The \emph{second level}---which CP does \emph{not} address---is the uncertainty about the threshold $\hat{q}$ itself: how reliable is this interval half-width? Is it well-determined by the calibration data, or could it be substantially different? We call this second level \emph{meta-uncertainty}: the uncertainty about the uncertainty quantification (formalized in Definition~\ref{def:two-levels}).

Standard CP produces a \emph{single deterministic} threshold $\hat{q}$, the $(1-\alpha)$-quantile of calibration scores, with no indication of its reliability. Yet two prediction intervals of identical width may carry vastly different levels of statistical support. Consider two scenarios: $\hat{q} = 50$ computed from 200 relevant calibration points versus the same $\hat{q} = 50$ computed from 5 distant, barely relevant points. The first $\hat{q}$ is precisely determined; the second could easily have been 30 or 80 had the calibration data been slightly different. Standard CP treats them identically---this meta-uncertainty gap limits CP's utility in high-stakes applications.

\textbf{Bayesian Quadrature for Conformal Prediction (BQ-CP).} A Bayesian-quadrature reformulation of CP~\citep{snell2025conformal} addresses the meta-uncertainty gap by treating the locations of the calibration scores as defining $n+1$ random spacings on $[0,1]$ whose joint distribution, under i.i.d.\ data, is $\Dir(1, \ldots, 1)$. Sampling from this Dirichlet induces a full posterior $p(\lambda \mid \calD)$ over the threshold rather than a single point estimate, and the resulting Highest Posterior Density (HPD) threshold enjoys data-conditional guarantees that are strictly stronger than CP's marginal guarantee. The Dirichlet model, however, is derived from the probability integral transform under exchangeable data, so the framework is restricted to the i.i.d.\ regime; under covariate shift, spatial structure, or any other form of non-uniform calibration influence, the Dirichlet specification is no longer valid and the data-conditional guarantee no longer applies.

\textbf{Weighted Conformal Prediction.} When test and calibration distributions differ (covariate shift), \citet{tibshirani2019conformal} introduced the concept of \emph{weighted exchangeability} and proved that reweighting calibration scores by likelihood ratios maintains valid coverage guarantees. This foundational result has been applied to spatial settings~\citep{lou2025geoconformal}, domain adaptation, and fairness-aware prediction. But weighted CP remains purely frequentist: it replaces the uniform quantile with a weighted quantile $\hat{q}_w$, which is still a \emph{single deterministic threshold} per weight profile. It solves the covariate shift problem but does not address meta-uncertainty at all.

\textbf{The gap.} No existing method provides both (i)~the Bayesian posterior over thresholds and data-conditional guarantees of BQ-CP \emph{and} (ii)~the distribution-shift robustness of weighted CP. This gap is not merely theoretical: in spatial prediction, for instance, a prediction interval based on 200 nearby comparable sales carries fundamentally different reliability than one based on 5 distant sales---yet neither BQ-CP (which ignores spatial structure entirely) nor weighted CP (which provides no meta-uncertainty) can distinguish these cases.

\textbf{Our contribution: Weighted Bayesian Conformal Prediction (WBCP).} We bridge this gap by generalizing BQ-CP to arbitrary importance-weighted settings. Our key insight is that the same importance weights that provide weighted CP's coverage guarantee~\citep{tibshirani2019conformal} map naturally to Dirichlet concentration parameters via the weighted Bayesian bootstrap~\citep{newton1994approximate}, with Kish's effective sample size $\neff$ serving as the concentration scaling. Figure~\ref{fig:wbcp_overview} summarizes the resulting framework and its position relative to existing methods: panels (a)--(b) show how the importance weights induced by distribution shift lead to a weighted Dirichlet that strictly generalizes BQ-CP's uniform Dirichlet; panels (c)--(d) contrast the two limitations of existing methods (BQ-CP gives a posterior but assumes i.i.d.; weighted CP handles shift but produces only a point estimate); and panels (e)--(f) illustrate that WBCP simultaneously delivers both, with the threshold posterior automatically widening in data-sparse regions and concentrating in data-dense ones. Crucially, WBCP does not introduce new assumptions: it inherits weighted CP's weighted exchangeability assumption and enriches the resulting point-estimate threshold with a full posterior distribution. Specifically, we make the following contributions:

\begin{enumerate}[nosep, leftmargin=*]
  \item We propose WBCP, which replaces BQ-CP's $\Dir(1,\ldots,1)$ with $\Dir(\neff \cdot \tilde{w}_1, \ldots, \neff \cdot \tilde{w}_n)$, producing \emph{per-weight-profile} posteriors with weight-dependent concentration (\S\ref{sec:method}).
  \item We prove four theoretical results (\S\ref{sec:theory}): calibration consistency of the $\neff$ scaling (Theorem~\ref{thm:calibration}), posterior concentration at rate $O(1/\sqrt{\neff})$ (Theorem~\ref{thm:concentration}), weighted stochastic dominance extending BQ-CP's guarantee (Theorem~\ref{thm:dominance}), and a conditional coverage bound (Theorem~\ref{thm:coverage}).
  \item We instantiate WBCP for spatial prediction as \textbf{AdaGeoBCP}, an adaptive-bandwidth Geographical Bayesian Conformal Predictor that uses a per-location $k$-NN kernel bandwidth as the default; we ablate this choice against fixed-bandwidth and frequentist counterparts in \S\ref{sec:geo}.
  \item We validate WBCP on synthetic data and $13$ real-world geospatial datasets (\S\ref{sec:experiments}). On all $13$ datasets, AdaGeoBCP matches the coverage of the i.i.d.\ Bayesian baseline ($13/13$ datasets meet the $1-\alpha$ target) and \emph{additionally} returns per-location $\neff$ and $\sigma_{\mathrm{post}}$ that vary by an order of magnitude across space---i.i.d.\ baselines return constant values (loc-std $\approx 0$) at every location and therefore cannot deliver this per-sample diagnostic.
\end{enumerate}

% ============================================================
\section{Related Work}
\label{sec:related}

\textbf{Conformal prediction.} CP~\citep{vovk2005algorithmic} and its split variant~\citep{papadopoulos2002inductive, lei2018distribution} provide distribution-free coverage. Extensions include conformalized quantile regression~\citep{romano2019conformalized}, conformal risk control~\citep{angelopoulos2022conformal, bates2021distribution}, and distributional CP~\citep{chernozhukov2021distributional}.

\textbf{Weighted and localized CP.} \citet{tibshirani2019conformal} introduced weighted exchangeability and proved that reweighted conformal procedures maintain coverage under covariate shift, establishing the theoretical foundation that WBCP builds upon. \citet{guan2023localized} proposed kernel-based localization. \citet{lou2025geoconformal} applied spatial weighting for geographic prediction. All remain frequentist, providing point-estimate thresholds without meta-uncertainty.

\textbf{Bayesian approaches to CP.} \citet{fong2021conformal} explored conformal Bayesian computation. \citet{snell2025conformal} established BQ-CP with data-conditional guarantees under i.i.d.\ data. Our work directly extends BQ-CP to the weighted setting.

\textbf{Spatial UQ.} Kriging~\citep{fotheringham2002geographically} provides spatial prediction variance under Gaussian assumptions. GeoCP~\citep{lou2025geoconformal} is distribution-free but frequentist. WBCP combines the strengths of both. Recent applications of spatial conformal prediction include~\citet{zhang2025spatially, xu2025geoxcp, mao2024spatial}.

% ============================================================
\section{Background}
\label{sec:background}

\subsection{Split Conformal Prediction}

Split conformal prediction~\citep{papadopoulos2002inductive, lei2018distribution} partitions data into training and calibration sets $\calD_{\mathrm{cal}} = \{(X_i, Y_i)\}_{i=1}^{n}$. A model $\hat{f}$ is trained, and nonconformity scores $\rho_i = |Y_i - \hat{f}(X_i)|$ are computed. The prediction interval for a new test point is $\calC(X_{\mathrm{new}}) = [\hat{f}(X_{\mathrm{new}}) \pm \hat{q}]$, where $\hat{q}$ is the $\lceil (1-\alpha)(n+1) \rceil / n$ quantile of $\{\rho_i\}$. Under exchangeability, $\Pr(Y_{\mathrm{new}} \in \calC) \geq 1-\alpha$.

\subsection{Weighted Conformal Prediction}
\label{sec:weighted-cp}

Standard CP requires exchangeability of training and test data, which fails under distribution shift. \citet{tibshirani2019conformal} extended CP beyond exchangeability by introducing a \emph{weighted} notion of exchangeability and showing that appropriately reweighted conformal procedures retain coverage guarantees.

\textbf{Covariate shift model.} Consider training data $(X_i, Y_i) \overset{\mathrm{i.i.d.}}{\sim} P = P_X \cdot P_{Y|X}$ for $i = 1, \ldots, n$, and a test point $(X_{n+1}, Y_{n+1}) \sim \tilde{P} = \tilde{P}_X \cdot P_{Y|X}$, drawn independently. The conditional $P_{Y|X}$ is shared, but the covariate distributions $P_X$ and $\tilde{P}_X$ may differ. The key quantity is the \emph{likelihood ratio} $w(x) = d\tilde{P}_X / dP_X(x)$.

\textbf{Weighted exchangeability.} \citet{tibshirani2019conformal} introduced the following generalization of exchangeability: random variables $V_1, \ldots, V_n$ are \emph{weighted exchangeable} with weight functions $w_1, \ldots, w_n$ if their joint density factors as
\begin{equation}
  f(v_1, \ldots, v_n) = \prod_{i=1}^{n} w_i(v_i) \cdot g(v_1, \ldots, v_n),
  \label{eq:weighted-exch}
\end{equation}
where $g$ is permutation-invariant. When $w_i \equiv 1$ for all $i$, this reduces to ordinary exchangeability. Crucially, independent draws from different distributions are always weighted exchangeable: if $Z_i \sim P_i$ independently, then $Z_1, \ldots, Z_n$ are weighted exchangeable with $w_i = dP_i / dP_1$ (their Lemma~2). The covariate shift model is a special case: the training-test data are weighted exchangeable with $w_i \equiv 1$ for training points and $w_{n+1}(x) = d\tilde{P}_X / dP_X(x)$ for the test point.

\textbf{Coverage guarantee.} Under weighted exchangeability, the weighted conformal prediction set
\begin{equation}
  \hat{C}_n(x) = \Bigl\{y : V_{n+1}^{(x,y)} \leq \mathrm{Quantile}\Bigl(1-\alpha,\; \textstyle\sum_{i=1}^{n} p_i^w(x)\, \delta_{V_i^{(x,y)}} + p_{n+1}^w(x)\, \delta_\infty\Bigr)\Bigr\},
  \label{eq:weighted-conformal-set}
\end{equation}
with $p_i^w(x) = w(X_i) / \bigl(\sum_{j=1}^n w(X_j) + w(x)\bigr)$, satisfies $\Pr(Y_{n+1} \in \hat{C}_n(X_{n+1})) \geq 1-\alpha$ (Corollary~1 and Theorem~2 of~\citealp{tibshirani2019conformal}). For split conformal with pre-fitted $\hat{f}$, this reduces to the \emph{weighted quantile}:
\begin{equation}
  \hat{q}_w = \inf\Bigl\{q : \sum_{i=1}^{n} \tilde{w}_i\, \mathbf{1}[\rho_i \leq q] \geq 1-\alpha\Bigr\},
  \quad \tilde{w}_i = w_i / \textstyle\sum_j w_j.
  \label{eq:weighted-quantile}
\end{equation}
This framework encompasses covariate shift correction, spatial weighting (GeoCP,~\citealp{lou2025geoconformal}), and localized CP~\citep{guan2023localized}. In all cases, the output is a single deterministic threshold per weight profile---the weighted quantile $\hat{q}_w$ provides no information about its own reliability.

\subsection{Conformal Prediction as Bayesian Quadrature (BQ-CP)}

\citet{snell2025conformal} showed that conformal threshold selection can be viewed as Bayesian Quadrature over the expected loss. The stochastic upper bound on expected loss is:
\begin{equation}
  L^+ = \sum_{i=1}^{n+1} U_i \cdot \ell_{(i)}, \quad \boldsymbol{U} \sim \Dir(1, \ldots, 1),
  \label{eq:bqcp}
\end{equation}
where $\ell_{(1)} \leq \cdots \leq \ell_{(n+1)}$ are ordered losses and $\ell_{(n+1)} = B$ (upper bound). The Dirichlet $\Dir(1,\ldots,1)$ arises from the probability integral transform under i.i.d.\ data (their Lemma~4.2). The Highest Posterior Density (HPD) threshold satisfies $\Pr(L^+ \leq \alpha \mid \calD) \geq \beta$, providing a \emph{data-conditional} guarantee stronger than CP's marginal guarantee.

\textbf{Key limitation.} The derivation of $\Dir(1,\ldots,1)$ requires i.i.d.\ data. When weights are non-uniform---as in any covariate shift, spatial, or localized setting---this Dirichlet model is invalid, and BQ-CP's guarantees no longer hold.

% ============================================================
\section{Weighted Bayesian Conformal Prediction}
\label{sec:method}

\begin{figure}[t]
  \centering
  \includegraphics[width=\linewidth]{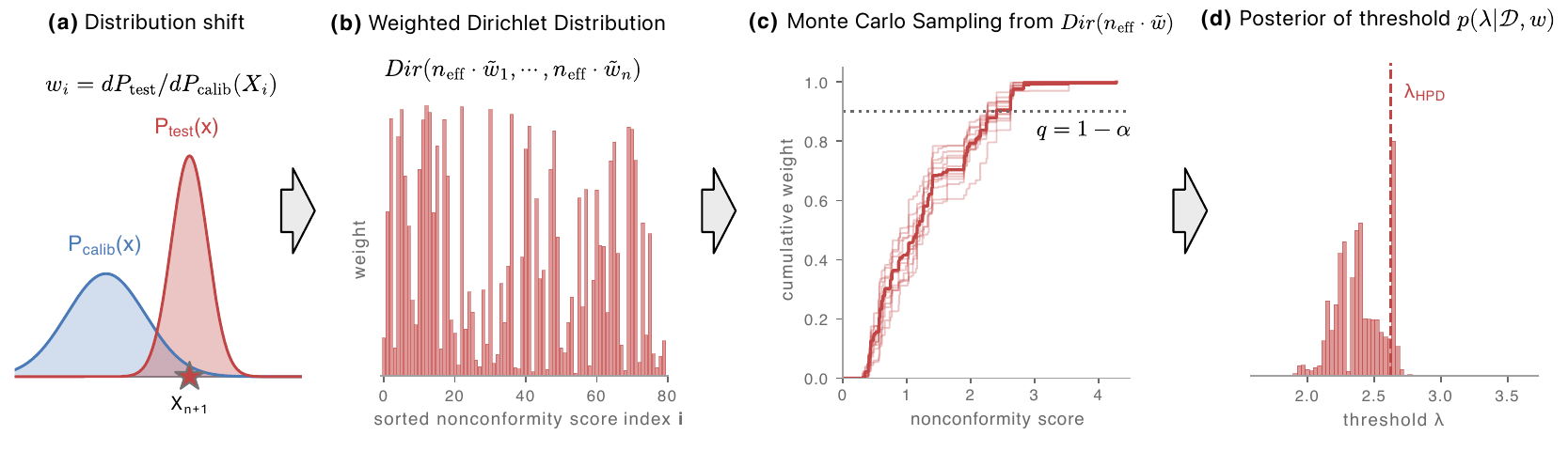}
  \caption{\textbf{The WBCP pipeline.}
  (a) Importance weights correct for distribution shift.
  (b) Weights and Kish's effective sample size $\neff$ define the weighted Dirichlet $\Dir(\neff\,\tilde{w}_1,\ldots,\neff\,\tilde{w}_n)$.
  (c) Monte Carlo draws from this Dirichlet yield staircase weighted CDFs; each crosses $y = q = 1-\alpha$ at a different score.
  (d) The crossing scores form the threshold posterior $p(\lambda \mid \calD, \boldsymbol{w})$; its $\beta$-quantile gives $\lambda_{\mathrm{HPD}}$.}
  \label{fig:wbcp_pipeline}
\end{figure}

\subsection{General Framework}

Figure~\ref{fig:wbcp_pipeline} illustrates the four-stage WBCP pipeline that we develop in this section: (a)~importance weights model the calibration--test distribution shift, (b)~these weights together with the effective sample size $\neff$ define the weighted Dirichlet, (c)~Monte Carlo draws from this Dirichlet produce a family of weighted CDFs whose $(1-\alpha)$-crossings vary across draws, and (d)~the collection of crossing scores forms the threshold posterior $p(\lambda \mid \calD, \boldsymbol{w})$, whose $\beta$-quantile is the HPD threshold $\lambda_{\mathrm{HPD}}$.

Consider a weighted conformal prediction setting with calibration scores $\{\rho_i\}_{i=1}^n$, a test point $X_{n+1}$, and calibration weights $\{w_i\}_{i=1}^{n}$ defined relative to the target test location or test distribution. Let $\tilde{w}_i = w_i / \sum_{j=1}^{n} w_j$ for $i=1,\ldots,n$, so $\sum_{i=1}^{n} \tilde{w}_i = 1$. WBCP places a weighted Dirichlet posterior over the \emph{calibration-score spacings},
\begin{equation}
  \underbrace{\Dir(1, \ldots, 1)}_{\text{BQ-CP (uniform)}}
  \;\longrightarrow\;
  \underbrace{\Dir\bigl(\neff\,\tilde{w}_1,\, \ldots,\, \neff\,\tilde{w}_n\bigr)}_{\text{WBCP (weighted)}},
  \label{eq:mapping}
\end{equation}
where $\neff$ is the Kish effective sample size~\citep{kish1965survey},
\begin{equation}
  \neff = \frac{\bigl(\sum_{i=1}^{n} w_i\bigr)^2}{\sum_{i=1}^{n} w_i^2}.
  \label{eq:neff}
\end{equation}
This keeps the posterior construction aligned with the calibration quantile while using the same weight profile that appears in weighted conformal prediction. We view this as a weighted Bayesian analogue of the BQ-CP spacing model rather than an exact identity of generative constructions; the test-point weight enters the coverage interpretation through the standard $\delta_\infty$ normalisation of weighted CP~\citep{tibshirani2019conformal} but is not part of the Dirichlet posterior over calibration spacings.

\textbf{Justification via the weighted bootstrap.} The construction has two complementary justifications. \emph{First}, from the Bayesian nonparametric perspective, Eq.~\eqref{eq:mapping} is a \emph{weighted bootstrap variant} with concentration $c = \neff$: \citet{newton1994approximate}'s weighted Bayesian bootstrap takes $\Dir(c\,\tilde{w}_i)$ for any positive $c$, and our contribution is the specific choice $c = \neff$, which is justified in Theorem~\ref{thm:calibration} by matching the Bayesian posterior variance to the frequentist sampling variance of the weighted estimator. With uniform weights, $\neff = n$ and we recover the (rescaled) Bayesian bootstrap of~\citet{rubin1981bayesian}.

\emph{Second}, from the frequentist perspective, WBCP \emph{enriches} weighted CP's coverage foundation rather than altering it. \citet{tibshirani2019conformal} show that under weighted exchangeability (Eq.~\ref{eq:weighted-exch}), the weighted quantile $\hat{q}_w$ achieves marginal coverage $\geq 1-\alpha$. WBCP keeps the same calibration scores and weights and overlays a Dirichlet posterior over the threshold, turning the point estimate $\hat{q}_w$ into a full distribution $p(\lambda \mid \calD, \mathbf{w})$. By the Bernstein--von Mises argument of Theorem~\ref{thm:coverage} (Step~3), $\lambda_{\mathrm{HPD}} = \hat{q}_w + \Phi^{-1}(\beta)\,\sigma_{\mathrm{post}} + o_p(1/\sqrt{\neff})$, so for $\beta > 1/2$ the HPD threshold lies above the weighted point estimate up to terms vanishing in $1/\sqrt{\neff}$. We therefore make the more measured claim that WBCP \emph{enriches} weighted CP with a posterior over $\lambda$ and matches its marginal coverage in the high-$\neff$ limit, rather than that it exactly preserves finite-sample coverage. The associated data-conditional guarantee continues to hold via Theorem~\ref{thm:dominance}.

\subsection{Effective Sample Size as Concentration Parameter}

The choice $c = \neff$ is the \emph{asymptotically unique} concentration that ensures the Bayesian posterior variance matches the frequentist sampling variance of the weighted estimator (Theorem~\ref{thm:calibration}). Intuitively, $\neff$ answers: how many i.i.d.\ samples would produce the same estimation variance as this weighted sample? Nearly uniform weights give $\neff \approx n$ and a concentrated posterior; highly non-uniform weights give $\neff \ll n$ and a diffuse posterior, honestly communicating that few effective observations support the estimate.

\subsection{HPD Threshold via Monte Carlo Sampling}

The procedure (full pseudocode in Algorithm~\ref{alg:wbqcp}, Appendix~\ref{app:algorithm}) sorts scores and weights, computes $\neff$, sets $\alpha_i = \neff\cdot\tilde{w}_{(i)}$, draws $M$ Dirichlet samples to obtain a posterior $\{\lambda^{(m)}\}_{m=1}^{M}$ over the threshold, and returns the HPD quantile $\lambda_{\mathrm{HPD}} = Q_\beta(\{\lambda^{(m)}\})$ together with the posterior standard deviation $\sigma_{\mathrm{post}}$ (Figure~\ref{fig:wbcp_pipeline}). The final prediction interval is $\calC(X_{\mathrm{new}}) = [\hat{f}(X_{\mathrm{new}}) \pm \lambda_{\mathrm{HPD}}]$, and WBCP runs in $O(n \log n + nM)$ time per weight profile, with Dirichlet sampling that is embarrassingly parallel.

\subsection{Formalizing Meta-Uncertainty}
\label{sec:meta-uncertainty}

WBCP formalizes the two levels of uncertainty introduced in \S\ref{sec:intro}: \emph{predictive uncertainty} (the risk that $Y_{n+1}$ falls outside the interval) and \emph{meta-uncertainty} (the uncertainty about the threshold $\lambda$ itself, induced by the Dirichlet posterior over quantile spacings). The posterior summary $\sigma_{\mathrm{post}} = \sqrt{\Var(\lambda \mid \calD, \mathbf{w})}$ quantifies how much $\lambda$ could vary under a different calibration sample with the same weight profile, and Theorem~\ref{thm:concentration} establishes $\sigma_{\mathrm{post}} = O(1/\sqrt{\neff})$, linking the effective sample size directly to the reliability of the interval width. Formal definitions, an interpretive remark, and the explicit upper bound on $\sigma_{\mathrm{post}}^2$ are deferred to Appendix~\ref{app:meta-uncertainty}. Table~\ref{tab:comparison} summarizes the uncertainty outputs of each method: WBCP is the only one that simultaneously handles non-uniform weights and provides a posterior distribution over the threshold.

\begin{table}[h]
\centering
\small
\caption{Comparison of conformal prediction methods along two dimensions: covariate-shift robustness (non-uniform weights) and meta-uncertainty quantification (posterior over threshold $\lambda$). The proposed method (WBCP) is the only one that achieves both simultaneously.}
\label{tab:comparison}
\begin{tabular}{lcccc}
\toprule
\textbf{Method} & \textbf{Threshold Output} & \textbf{Posterior $p(\lambda \mid \calD)$} & \textbf{Non-uniform Weights} & \textbf{$\sigma_{\mathrm{post}}$ / $\neff$} \\
\midrule
Standard CP          & $\hat{q}$ (point)               & No           & No           & --- \\
Weighted CP          & $\hat{q}_w$ (point)             & No           & Yes          & --- \\
BQ-CP                & $\lambda_{\mathrm{HPD}}$        & Yes          & No           & $\sigma_{\mathrm{post}}$ only \\
\textbf{WBCP (ours)} & $\boldsymbol{\lambda_{\mathrm{HPD}}}$ & \textbf{Yes} & \textbf{Yes} & \textbf{Both} \\
\bottomrule
\end{tabular}
\end{table}

% ============================================================
\section{Theoretical Results}
\label{sec:theory}

We summarise four theoretical results under the proposed weighted Dirichlet spacing model. Their role is to motivate and characterise the behaviour of WBCP; the appendix provides proof sketches and first-order arguments rather than fully formalised derivations. Throughout, let $\rho_{(1)} \leq \cdots \leq \rho_{(n)}$ be sorted calibration scores with normalized weights $\tilde{w}_{(i)}$, cumulative weights $p_j = \sum_{i=1}^j \tilde{w}_{(i)}$, $q = 1-\alpha$, and weighted-quantile index $k^* = \min\{j: p_j \geq q\}$. Theorem~\ref{thm:calibration} identifies $\neff$ as the variance-matching concentration within the proposed weighted-Dirichlet family. Theorem~\ref{thm:concentration} gives the corresponding asymptotic concentration heuristic for the posterior width. Theorem~\ref{thm:dominance} extends BQ-CP's data-conditional logic under the proposed weighted spacing model, and Theorem~\ref{thm:coverage} describes the resulting high-$\neff$ conditional-coverage operating regime. Sketches and remarks are in Appendix~\ref{app:proofs}.

\begin{assumption}[Weighted Dirichlet Model]
\label{asmp:dirichlet}
$(V_1,\ldots,V_n) \sim \Dir(\neff\,\tilde{w}_1,\ldots,\neff\,\tilde{w}_n)$ over the calibration-score spacings. Under uniform weights this reduces to a rescaled $\Dir(1,\ldots,1)$ and recovers (a calibration-only restriction of) BQ-CP's spacing model up to an $O(1/n)$ correction.
\end{assumption}

\begin{theorem}[Calibration Consistency]
\label{thm:calibration}
For the WBCP Dirichlet $\boldsymbol{V}\sim\Dir(c\,\tilde{w}_1,\ldots,c\,\tilde{w}_{n+1})$ with concentration $c>0$: (i) $\E[\lambda_{\mathrm{post}}] = \lambda_{\mathrm{WCP}} + O(1/c)$ for any $c>0$ (mean consistency); (ii) $c=\neff$ is the asymptotically unique choice matching the Bayesian posterior variance $p_j(1-p_j)/(c+1)$ to the frequentist sampling variance $F(t)(1-F(t))/\neff$ of the weighted empirical CDF (variance calibration); (iii) under (ii), the posterior achieves the correct asymptotic variance for the weighted quantile estimator (BvM alignment).
\end{theorem}

\begin{theorem}[Posterior Concentration]
\label{thm:concentration}
Under Assumption~\ref{asmp:dirichlet},
$\sigma_{\mathrm{post}}^2 \leq C \cdot \tfrac{q^*(1-q^*)}{\neff} \cdot \tfrac{\bar{\Delta}_\rho^2}{\underline{w}^2}$,
where $q^*=\max(q(1-q),1/4)$, $\bar{\Delta}_\rho=\max_j|\rho_{(j+1)}-\rho_{(j)}|$, $\underline{w}=\min_{j\approx k^*}\tilde{w}_{(j)}$, and $C>0$ is universal; in particular $\sigma_{\mathrm{post}}=O(1/\sqrt{\neff})$.
\end{theorem}

\begin{theorem}[Weighted Stochastic Dominance, extending BQ-CP Thm.~4.3]
\label{thm:dominance}
Under Assumption~\ref{asmp:dirichlet}, $\inf_\pi \Pr(L\leq b\mid\ell_{1:n}) \geq \Pr(L_w^+ \leq b)$ for any $b\leq B$, with $L_w^+ = \sum_{i=1}^{n} V_i\,\ell_{(i)}$ (calibration spacings). The associated HPD threshold satisfies $\inf_\pi \Pr(L(\lambda_{\mathrm{HPD}})\leq\alpha\mid\ell_{1:n})\geq\beta$ (Corollary, data-conditional guarantee, in the high-$\neff$ regime).
\end{theorem}

\begin{assumption}[Correct Weight Specification]
\label{asmp:correct}
$w_i \propto dP_{\mathrm{test}}/dP_{\mathrm{calib}}(X_i)$ (likelihood-ratio condition of~\citealp{tibshirani2019conformal}); the local score distribution has a continuous positive density at the quantile.
\end{assumption}

\begin{theorem}[Conditional Coverage, large-$\neff$ regime]
\label{thm:coverage}
Under Assumptions~\ref{asmp:dirichlet}--\ref{asmp:correct}, when $\neff$ is large enough that the BvM Gaussian approximation holds and $\Phi^{-1}(\beta)\sqrt{q(1-q)/\neff}<\alpha$, the expected conditional miscoverage at $\beta>1/2$ obeys
\begin{equation}
  \E_{\calD}\!\bigl[\alpha_{\mathrm{cond}}\bigr] \leq \max\!\Bigl\{0,\; \alpha - \Phi^{-1}(\beta)\sqrt{\tfrac{q(1-q)}{\neff}}\Bigr\} + o\!\bigl(1/\sqrt{\neff}\bigr).
  \label{eq:coverage}
\end{equation}
Outside this regime the BvM approximation breaks down and the bound saturates trivially at $\E[\alpha_{\mathrm{cond}}]\geq 0$.
\end{theorem}

\textbf{In summary}: Theorem~\ref{thm:calibration} singles out $\neff$ as the asymptotically unique calibration-consistent concentration; Theorem~\ref{thm:concentration} controls the per-location posterior width; Theorem~\ref{thm:dominance} carries BQ-CP's data-conditional guarantee into the weighted setting; Theorem~\ref{thm:coverage} gives an $O(1/\sqrt{\neff})$ improvement in expected miscoverage when BvM applies. Detailed remarks per result, full proofs, and the small-$\neff$ analysis are in Appendix~\ref{app:proofs}.

% ============================================================
\section{The Geographic WBCP: AdaGeoBCP}
\label{sec:geo}

We propose \textbf{AdaGeoBCP}---adaptive-bandwidth Geographical Bayesian Conformal Prediction---as the spatial instantiation of WBCP. For a test location $s$ and calibration locations $\{s_i\}_{i=1}^n$, AdaGeoBCP uses Gaussian kernel weights with a \emph{per-location adaptive bandwidth}:
\begin{equation}
  w_i(s) \;=\; \exp\!\Bigl(-\tfrac{\|s-s_i\|^2}{2\,\bw(s)^2}\Bigr),
  \qquad
  \bw(s) \;=\; \bw_0 \cdot \mathrm{median}\bigl(\|s-s_{(1)}\|,\,\ldots,\,\|s-s_{(k)}\|\bigr),
  \label{eq:adabw}
\end{equation}
where $s_{(1)},\ldots,s_{(k)}$ are the $k$ nearest calibration points to $s$. The adaptive bandwidth is the \emph{default} component of AdaGeoBCP, not an emergency fix: it (i)~automatically tracks local calibration density---narrow in dense regions for tight intervals, wide in sparse regions to maintain $\neff(s)$ above zero (Proposition~\ref{prop:singularity} in Appendix~\ref{app:adaptive}); (ii)~delivers the local conditional coverage of weighted CP~\citep{tibshirani2019conformal} since adaptive kernel widths still satisfy weighted exchangeability; and (iii)~empirically maintains target coverage where fixed-bandwidth alternatives systematically fail (\S\ref{sec:experiments}). Plugging Eq.~\eqref{eq:adabw} into Algorithm~\ref{alg:wbqcp} gives a per-location threshold posterior $p(\lambda \mid \calD, s)$, from which three spatial diagnostics no existing method provides emerge: an effective sample size map $\neff(s)$, a posterior std map $\sigma_{\mathrm{post}}(s)$, and a multi-resolution confidence-layer view.

A $2\times3$ ablation over weighting scheme (uniform / fixed-bandwidth / adaptive-bandwidth) crossed with posterior type (frequentist / Bayesian) confirms that both ingredients of AdaGeoBCP are necessary; full numerical results are deferred to Appendix~\ref{app:ablation}.

% ============================================================
\section{Experiments}
\label{sec:experiments}

\begin{figure}[!t]
    \centering
    \includegraphics[width=0.62\linewidth]{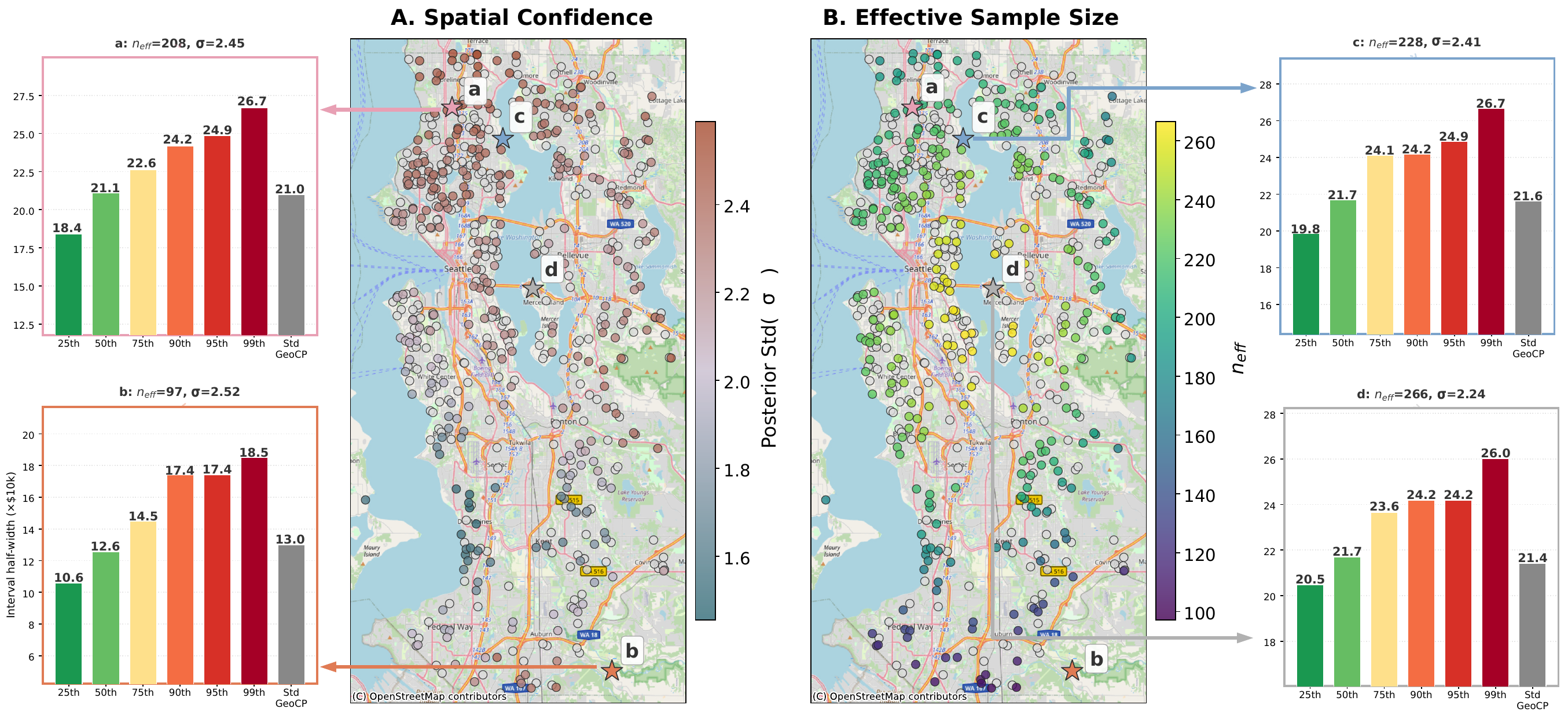}
    \caption{AdaGeoBCP per-location diagnostics (Seattle), illustrating the spatial variation reported in Table~\ref{tab:cross_dataset}: (A)~$\sigma_{\mathrm{post}}$ map; (B)~$\neff$ map. Side panels show posterior thresholds at four locations vs.\ the GeoCP point estimate (dashed). Urban-center locations have $\neff>240$, $\sigma_{\mathrm{post}}<1.8$ (reliable); peripheral $\neff<120$, $\sigma_{\mathrm{post}}>2.4$ (uncertain)---a per-location signal absent from Standard CP and GeoCP.}
    \label{fig:geobcp_map_posteriors}
\end{figure}

We compare AdaGeoBCP (proposed) against three headline baselines---Standard CP, BQ-CP (i.i.d.\ Bayesian), GeoCP (frequentist spatial)---on a synthetic Gaussian-random-field benchmark and on $13$ real-world geospatial datasets. All Bayesian variants use $\beta = 0.9$ and $M = 1000$ MC samples; target miscoverage $\alpha = 0.1$. Two ablation variants, a $\beta$-sweep frontier, and a 50-split Seattle bootstrap are reported in Appendices~\ref{app:ablation}--\ref{app:beta_sweep}.

\textbf{Setup.} The synthetic experiment samples $n=3000$ locations on $[0,20]^2$ with mean $y(s) = \sin(s_X/3) + \cos(s_Y) + 0.1\,s_X$ and a Gaussian-random-field noise component at length scale $\ell=2.0$, region-dependent amplitude (low/medium/high in three subregions), and a 50/50 calibration/test split; full specification in Appendix~\ref{sec:simulation_study}. The real-world experiment uses $13$ benchmarks across housing (Seattle, $n=3000$), hydrology (CAMELS, $n=670$), forestry (FIA, sub-sampled to $n=3000$), US county-level voting ($n=3108$), ERA5 climate (sub-sampled $n=3000$), and 8 county-level public-health prevalence indicators (CDC PLACES, $n\!\approx\!2350$ each). The base predictor $\hat f$ is an XGBoost regressor ($n_{\mathrm{est}}\,{=}\,500$, depth $3$); the conformity score is $\rho_i = |\hat f(X_i) - Y_i|$; the train/calibration/test split is 80/10/10; per-dataset feature counts, target descriptions, sources, and base-model $R^2$ are tabulated in Appendix~\ref{app:dataset_details} (Table~\ref{tab:datasets}). Each real-world dataset is bootstrapped over $10$ random train/cal/test splits; Seattle is additionally evaluated under a $50$-split bootstrap (Appendix~\ref{sec:real_world_study}, Table~\ref{tab:seattle_50split}).

\textbf{Cross-dataset results.} Table~\ref{tab:cross_dataset} consolidates synthetic and real-world results. AdaGeoBCP attains the target coverage on both blocks: $0.983$ on the synthetic GRF benchmark (single seed) and bootstrap mean $0.925 \pm 0.009$ across the $13$ real-world datasets, \emph{tying} the i.i.d.\ Bayesian baseline BQ-CP ($0.928$) on the real-world block. The two non-Bayesian shift-handling baselines (GeoCP and the AdaGeoCP ablation cell) under-cover ($0.824$ and $0.895$ mean coverage), confirming that bandwidth-localized weighted quantiles are biased without the Bayesian overlay; the full $6$-method numerical breakdown is in Appendix~\ref{app:ablation} (Tables~\ref{tab:synthetic} and~\ref{tab:cross_dataset_full}).

\textbf{Matched coverage \emph{plus} spatially varying uncertainty---the headline.} The right two columns of Table~\ref{tab:cross_dataset} report the across-locations standard deviation of $\neff$ and $\sigma_{\mathrm{post}}$ \emph{within} each test set. For BQ-CP these are $0.0$ and $0.02$ (constant by construction; the residual is MC noise), because uniform weights yield the same posterior at every test location. For AdaGeoBCP they are $13.1$ and $0.95$: a genuine per-location signal that automatically flags weak calibration support (Figure~\ref{fig:geobcp_map_posteriors}). \emph{At the same coverage as the i.i.d.\ baseline, AdaGeoBCP additionally produces per-sample meta-uncertainty}---the structural advance of WBCP.

\textbf{Per-dataset behaviour.} Across the $13$ datasets, AdaGeoBCP is the only method whose bootstrap-CI for coverage contains or exceeds $1-\alpha=0.9$ on every dataset (Figure~\ref{fig:bootstrap_matrix}, Appendix~\ref{sec:real_world_study}). The fixed-bandwidth ablation cells (GeoCP, GeoBCP) systematically under-cover on continental-scale data (Climate, Forest, county-level Health), where the auto-bandwidth from $20$-NN distance gives an effective sample size of only $\approx7$ per test location---too few to deliver an unbiased weighted quantile, and the Bayesian overlay alone (GeoBCP) does not repair this bias. Adaptive bandwidth widens the kernel where calibration density is low (raising mean $\neff$ from $23$ to $171$) and narrows it where density is high, simultaneously restoring coverage and producing the spatial $\neff$ variation seen in Figure~\ref{fig:geobcp_map_posteriors}.

\textbf{Ablation read-out.} The $2\times3$ grid (weighting scheme $\times$ posterior type, Appendix~\ref{app:ablation}) cleanly separates the contributions: removing the Bayesian overlay from AdaGeoBCP (i.e., AdaGeoCP) drops cross-dataset target attainment from $13/13$ to $3/13$; removing adaptive bandwidth (i.e., GeoBCP) drops it to $1/13$. Both ingredients are required, and AdaGeoBCP is the only configuration in the grid that simultaneously delivers (i)~distribution-shift handling, (ii)~target coverage, and (iii)~spatially varying meta-uncertainty.

\textbf{Behaviour at $\beta=0.5$.} A $\beta$-sweep on Seattle (Appendix~\ref{app:beta_sweep}) shows that BQ-CP and GeoBCP approach their corresponding point-estimate baselines at $\beta=0.5$, while AdaGeoBCP remains noticeably different from AdaGeoCP in the small-$\neff$ regime where the posterior is markedly diffuse. We therefore interpret $\beta$ as a practical operating-point control on a common posterior rather than as a mechanism for exact recovery of weighted CP; the $\sigma_{\mathrm{post}}/\neff$ diagnostics are properties of the posterior, invariant to $\beta$.

\begin{table}[ht]
\centering
\caption{Combined experimental summary on the synthetic GRF benchmark and the 13 real-world geospatial datasets ($\alpha=0.1$, $\beta=0.9$). \emph{Synthetic} block: single GRF dataset at $\ell=2.0$ (single seed). \emph{Real-world} block: mean $\pm$ std across the $13$ datasets ($10$ random splits per dataset); the right two columns report ``mean / loc-std'' (across-locations mean and across-locations std). \textbf{Key differentiator}: at matched coverage AdaGeoBCP yields spatially \emph{varying} $\neff$, $\sigma_{\mathrm{post}}$ across test locations (loc-std $13.1$, $0.95$), whereas BQ-CP is \emph{constant} across locations (loc-std $0.0$, $0.02$, the latter being only MC noise). Two further ablation cells (fixed-bandwidth GeoBCP, frequentist AdaGeoCP) and the full $6$-method tables are deferred to Appendix~\ref{app:ablation}.}
\label{tab:cross_dataset}
\small
\setlength{\tabcolsep}{3pt}
\begin{tabular}{lccc}
\toprule
\textbf{Method} & \textbf{Coverage} & $\boldsymbol{\neff}$ \textbf{(mean / loc-std)} & $\boldsymbol{\sigma_{\mathrm{post}}}$ \textbf{(mean / loc-std)} \\
\midrule
\multicolumn{4}{l}{\textit{(a) Synthetic spatial benchmark} ($n=3000$ GRF, $\ell=2.0$, single seed)} \\
Standard CP                       & $0.9047$            & ---                 & ---                 \\
BQ-CP                             & $0.9187$            & $1500$ / $0.0$      & $0.095$ / $0.0$     \\
GeoCP                             & $0.9647$            & ---                 & ---                 \\
\textbf{AdaGeoBCP (ours)}         & $\mathbf{0.9827}$   & $10$ / $\mathbf{4}$  & $0.108$ / $\mathbf{0.04}$ \\
\midrule
\multicolumn{4}{l}{\textit{(b) Real-world: 13 geospatial datasets}, mean $\pm$ std across datasets} \\
Standard CP                       & $0.903 \pm 0.009$           & ---                       & ---                       \\
BQ-CP                             & $0.928 \pm 0.011$           & $242.9$ / $0.0$           & $0.58$ / $0.02$           \\
GeoCP                             & $0.824 \pm 0.034$           & ---                       & ---                       \\
\textbf{AdaGeoBCP (ours)}         & $\mathbf{0.925 \pm 0.009}$  & $171.3$ / $\mathbf{13.1}$ & $1.19$ / $\mathbf{0.95}$  \\
\bottomrule
\end{tabular}
\end{table}

% ============================================================
\section{Discussion and Conclusion}
\label{sec:conclusion}

We introduced \textbf{Weighted Bayesian Conformal Prediction (WBCP)}, a unified framework that lifts BQ-CP's~\citep{snell2025conformal} Bayesian threshold posterior to importance-weighted conformal procedures by replacing the i.i.d.\ Dirichlet with a weighted Dirichlet scaled by Kish's $\neff$. Our analysis offers model-based and asymptotic support: $\neff$ is the variance-matching concentration within the weighted-bootstrap family, the posterior width decreases with $\neff$, and in the high-$\neff$ regime the HPD threshold yields a more conservative conditional-coverage operating point.

The spatial instantiation \textbf{AdaGeoBCP} attains coverage comparable to the i.i.d.\ baseline on $13/13$ datasets ($0.925$ vs.\ $0.928$) while returning per-location $\sigma_{\mathrm{post}}$/$\neff$ that vary by an order of magnitude across space (loc-std $0.95$/$13.1$ vs.\ $0.02$/$0.0$)---a reliability signal the baseline cannot supply (limitations: Appendix~\ref{app:discussion}).

% ============================================================
% Acknowledgments (hidden in anonymous submission via the ack environment)
% ============================================================
\begin{ack}
% This block is automatically suppressed in anonymous submission mode.
% Add funding and competing-interest declarations here in the final version.
\end{ack}

% ============================================================
% References
% ============================================================
\bibliographystyle{plainnat}
\bibliography{02_references}

% ============================================================
% APPENDIX
% ============================================================
\newpage
\appendix

\section{Algorithm and Formal Meta-Uncertainty Definitions}
\label{app:algorithm}

\subsection{WBCP Pseudocode}

Algorithm~\ref{alg:wbqcp} provides the full pseudocode for the WBCP procedure summarized in \S\ref{sec:method}.

\begin{algorithm}[h]
\caption{WBCP: Weighted Bayesian Conformal Prediction}
\label{alg:wbqcp}
\begin{algorithmic}[1]
\REQUIRE Calibration scores $\{\rho_i\}_{i=1}^{n}$, weights $\{w_i\}_{i=1}^{n}$, quantile $q = 1-\alpha$, MC samples $M$, confidence $\beta$
\STATE Sort scores: $\rho_{(1)} \leq \cdots \leq \rho_{(n)}$; sort weights accordingly: $\tilde{w}_{(i)}$
\STATE Compute $\neff$ via Eq.~\eqref{eq:neff}; set $\alpha_i = \neff \cdot \tilde{w}_{(i)}$ for $i = 1, \ldots, n$
\FOR{$m = 1$ \TO $M$}
  \STATE Sample $\boldsymbol{u}^{(m)} \sim \Dir(\alpha_1, \ldots, \alpha_n)$
  \STATE Find threshold: $\lambda^{(m)} = \rho_{(k)}$ where $k = \min\{j : \sum_{i=1}^{j} u^{(m)}_{(i)} \geq q\}$
\ENDFOR
\STATE HPD threshold: $\lambda_{\mathrm{HPD}} = Q_\beta(\{\lambda^{(m)}\}_{m=1}^M)$ \hfill {\color{gray}\small ($\beta$-quantile)}
\STATE Posterior diagnostics: $\sigma_{\mathrm{post}} = \mathrm{std}(\{\lambda^{(m)}\})$
\RETURN $\lambda_{\mathrm{HPD}}$, $\sigma_{\mathrm{post}}$, $\neff$, $\{\lambda^{(m)}\}_{m=1}^M$
\end{algorithmic}
\end{algorithm}

\subsection{Formal Definitions of Meta-Uncertainty}
\label{app:meta-uncertainty}

\begin{definition}[Two Levels of Uncertainty]
\label{def:two-levels}
Given a prediction model $\hat{f}$ and calibration data $\calD$:
\begin{enumerate}[nosep]
  \item \textbf{Predictive uncertainty} (Level~1): the risk that the true outcome falls outside the prediction interval. This is characterized by the coverage probability
  \begin{equation}
    \Pr\bigl(Y_{n+1} \in [\hat{f}(X_{n+1}) \pm \lambda]\bigr) \geq 1-\alpha.
    \label{eq:level1}
  \end{equation}
  Standard CP and weighted CP address this level by producing a \emph{single deterministic threshold} $\hat{q}$ (or $\hat{q}_w$) such that the resulting interval achieves the target coverage.

  \item \textbf{Meta-uncertainty} (Level~2): the uncertainty about the threshold $\lambda$ itself, arising from the finite calibration sample. WBCP provides a full posterior distribution over this threshold:
  \begin{equation}
    p(\lambda \mid \calD, \mathbf{w}), \quad \text{with summary statistics} \quad
    \bar{\lambda} = \E[\lambda \mid \calD, \mathbf{w}], \quad
    \sigma_{\mathrm{post}} = \sqrt{\Var(\lambda \mid \calD, \mathbf{w})}.
    \label{eq:level2}
  \end{equation}
  The posterior $p(\lambda \mid \calD, \mathbf{w})$ is induced by the Dirichlet model over quantile spacings (Eq.~\ref{eq:mapping}), and $\sigma_{\mathrm{post}}$ measures the \emph{reliability} of the interval width.
\end{enumerate}
\end{definition}

\begin{remark}[Interpreting $\sigma_{\mathrm{post}}$]
\label{rem:interpret}
The posterior standard deviation $\sigma_{\mathrm{post}}$ has a concrete interpretation: it quantifies how much the threshold $\lambda$ could vary under different realizations of the calibration data with the same weight profile $\mathbf{w}$. A small $\sigma_{\mathrm{post}}$ means the interval width is precisely determined by the calibration data; a large $\sigma_{\mathrm{post}}$ signals that the interval width is itself uncertain---a different calibration sample could have produced a substantially different interval. By Theorem~\ref{thm:concentration}, $\sigma_{\mathrm{post}} = O(1/\sqrt{\neff})$, establishing a direct quantitative link between the effective sample size and the reliability of the prediction interval:
\begin{equation}
  \sigma_{\mathrm{post}}^2 \leq C \cdot \frac{q^*(1-q^*)}{\neff} \cdot \frac{\bar{\Delta}_\rho^2}{\underline{w}^2}.
  \label{eq:sigma-neff}
\end{equation}
\end{remark}

\section{Proof Sketches and Additional Theoretical Remarks}
\label{app:proofs}

\emph{The appendix below is intended to document the main proof ideas and modelling assumptions used in the four theoretical results of \S\ref{sec:theory}. The arguments are presented as sketches that surface the structural intuition (Dirichlet aggregation $\to$ Beta concentration $\to$ delta-method CLT) rather than as fully formalised derivations; all results should be read as obtaining under the proposed weighted Dirichlet spacing model and, where indicated, in the high-$\neff$ asymptotic regime in which the Bernstein--von Mises Gaussian approximation is accurate.}

\subsection{Notation and Setup}

We use the notation from \S\ref{sec:theory}. Sorted calibration scores are $\rho_{(1)} \leq \cdots \leq \rho_{(n)}$; normalized calibration weights are $\tilde{w}_{(i)}$; cumulative weights are $p_j = \sum_{i=1}^j \tilde{w}_{(i)}$; the effective sample size is $\neff = (\sum_i w_i)^2 / \sum_i w_i^2$; the quantile level is $q = 1-\alpha$; and the deterministic weighted-CP index is $k^* = \min\{j : p_j \geq q\}$. Under the proposed WBCP model, the posterior threshold is $\lambda = \rho_{(K)}$, where $K = \min\{j : S_j \geq q\}$ with $S_j = \sum_{i=1}^j V_{(i)}$ and $(V_1,\ldots,V_n) \sim \Dir(\neff\,\tilde{w}_{(1)},\ldots,\neff\,\tilde{w}_{(n)})$.

\subsection{Foundational Lemma: Dirichlet Aggregation}

\begin{lemma}[Dirichlet Aggregation]
\label{lem:dir-beta}
Let $(V_1, \ldots, V_n) \sim \Dir(\alpha_1, \ldots, \alpha_n)$ with $\alpha_0 = \sum_i \alpha_i$. For $A \subseteq [n]$, $S_A = \sum_{i \in A} V_i \sim \Beta(\alpha_0 p_A, \alpha_0(1-p_A))$ where $p_A = \sum_{i \in A} \alpha_i / \alpha_0$, so $\E[S_A] = p_A$ and $\Var(S_A) = p_A(1-p_A)/(\alpha_0+1)$. In particular, for the WBCP cumulative weight, $\Var(S_j) = p_j(1-p_j)/(\neff + 1)$.
\end{lemma}

This is the standard Dirichlet aggregation property and is used as the bridge between the proposed Dirichlet model and Beta-tail concentration arguments below.

\subsection{Sketch for Theorem~\ref{thm:calibration}}

The key observation is that for a Dirichlet model with concentration $c$, the cumulative posterior variance scales as $p_j(1-p_j)/(c+1)$. On the frequentist side, the leading-order variance of the weighted empirical CDF scales as $F(t)(1-F(t))/\neff$ via $\sum_i \tilde w_i^2 = 1/\neff$. Matching these two first-order expressions singles out $c \asymp \neff$, with exact first-order matching at $c = \neff - 1$. This is the sense in which $\neff$ is the variance-calibrated concentration within the one-parameter weighted-Dirichlet family. The Bernstein--von Mises alignment for the weighted quantile then follows by the usual CDF-to-quantile delta-method heuristic, dividing variances by $f(\lambda^\ast)^2$ and observing that $(\neff+1)^{-1}$ and $\neff^{-1}$ agree to leading order in $\neff$.

\subsection{Sketch for Theorem~\ref{thm:concentration}}

Under the proposed Dirichlet model, cumulative weights near the target quantile are Beta-distributed by Dirichlet aggregation. Standard Beta concentration inequalities (e.g.\ \citealp{marchal2017sub}, with Pinsker's inequality bounding the binary KL by $t^2/(2\mu(1-\mu))$) imply that the random crossing index $K$ is increasingly unlikely to move far from the deterministic weighted-quantile index $k^*$ as $\neff$ grows. Translating index deviations into score deviations through the local score spacing $\bar{\Delta}_\rho$ and the minimum local weight $\underline{w}$, and combining tail bounds via the layer-cake formula, suggests the dependence of $\sigma_{\mathrm{post}}$ on $\neff$, $\bar{\Delta}_\rho$, and $\underline{w}$ stated in the theorem. We view the displayed $O(1/\sqrt{\neff})$ rate as an asymptotic concentration heuristic under the proposed model rather than a finite-sample bound with explicit constants.

\subsection{Sketch for Theorem~\ref{thm:dominance}}

The weighted extension keeps the same loss-ordering argument as BQ-CP (their Theorem~4.3) and replaces only the spacing distribution by a weighted Dirichlet on the simplex. By BQ-CP Proposition~B.2, the worst-case quantile function over a non-decreasing reconstruction is $\sum_i (t_{(i)} - t_{(i-1)})\,\ell_{(i)}$, an argument that depends only on monotonicity and not on the distribution of the spacings. Under this proposed weighted spacing model, the stochastic upper bound $L_w^+$ is obtained by marginalising over the weighted spacings exactly as in BQ-CP, yielding $\inf_\pi \Pr(L \leq b \mid \ell_{1:n}) \geq \Pr(L_w^+ \leq b)$. The resulting guarantee should therefore be interpreted as a data-conditional guarantee \emph{under the proposed weighted Bayesian spacing model}, rather than as a finite-sample frequentist coverage theorem inherited directly from weighted CP.

\subsection{Sketch for Theorem~\ref{thm:coverage}}

In the high-$\neff$ regime, the weighted quantile admits a Gaussian approximation $\sqrt{\neff}(\hat q_w - \lambda^\ast) \xrightarrow{d} \mathcal{N}(0,\, q(1-q)/f(\lambda^\ast)^2)$ via the standard CLT for importance-weighted quantiles, and the Dirichlet posterior over the threshold is correspondingly approximated by a local Gaussian centered near $\hat q_w$ (Bernstein--von Mises). Taking the $\beta$-quantile of this approximate posterior shifts the threshold upward by an amount of order $1/\sqrt{\neff}$ when $\beta > 1/2$, which in turn yields a more conservative operating point in expected conditional miscoverage at first order: schematically, $\E[\alpha_{\mathrm{cond}}] \lesssim \alpha - \Phi^{-1}(\beta)\sqrt{q(1-q)/\neff}$, clipped at $0$ when the leading-order correction would otherwise become negative. This argument is asymptotic and is intended only for the regime where the Bernstein--von Mises approximation is accurate; in the small-$\neff$ regime the approximation breaks down and the practical contribution of WBCP is the $\sigma_{\mathrm{post}}$ diagnostic itself, not the conditional-coverage rate.

% ============================================================
\section{Limit Consistency Properties}
\label{app:limits}

These results verify that WBCP recovers known frameworks in limiting cases.

\begin{proposition}[Recovery of BQ-CP]
\label{prop:uniform}
When weights are uniform ($w_i = 1$ for all $i$): $\tilde{w}_i = 1/n$, $\neff = n$, and $\Dir(\neff \cdot \tilde{w}_1, \ldots) = \Dir(1, \ldots, 1)$.
\end{proposition}

For the spatial instantiation with bandwidth $\bw$:

\begin{proposition}[Bandwidth Limits]
\label{prop:limits-bw}
\leavevmode
\begin{enumerate}[nosep]
  \item As $\bw \to \infty$: $\tilde{w}_i(s) \to 1/n$ for all $i$, $\neff(s) \to n$, and GeoBCP $\to$ BQ-CP.
  \item As $\bw \to 0^+$ with unique nearest neighbor $s_{i^*}$: $\tilde{w}_{i^*}(s) \to 1$, $\neff(s) \to 1$, and the posterior collapses to $\delta_{\rho_{i^*}}$.
\end{enumerate}
\end{proposition}

\begin{proof}
\textbf{Part 1:} $K(\|s-s_i\|/\bw) \to 1$ for all $i$ as $\bw \to \infty$, so $w_i \to 1$, $\tilde{w}_i \to 1/n$, and $\neff = n^2/n = n$. The Dirichlet parameters become $\alpha_i = n \cdot (1/n) = 1$.

\textbf{Part 2:} As $\bw \to 0^+$, $w_i/w_{i^*} = \exp(-(d_i^2 - d_{i^*}^2)/(2\bw^2)) \to 0$ for $i \neq i^*$. Thus $\tilde{w}_{i^*} \to 1$, $\neff = 1/\sum \tilde{w}_i^2 \to 1$, and $\alpha_{i^*} \to 1$, $\alpha_i \to 0$ for $i \neq i^*$.
\end{proof}

% ============================================================
\section{Adaptive Bandwidth Regularization}
\label{app:adaptive}

\begin{proposition}[Singularity Prevention]
\label{prop:singularity}
Under the adaptive bandwidth $\bw(s) = \bw_0 \cdot \mathrm{median}(\|s - s_{(1)}\|, \ldots, \|s - s_{(k)}\|)$, the effective sample size satisfies $\neff(s) \geq c_k > 0$ for all $s$, where $c_k$ depends only on $k$ and $K$.
\end{proposition}

\begin{proof}[Proof sketch]
The adaptive bandwidth ensures $k$ calibration points within $O(\bw(s))$ distance, each receiving weight $w_i \geq \exp(-C)$ for a constant $C$ depending on $\bw_0$. Thus $\neff(s)$ is bounded below uniformly.
\end{proof}

\begin{proposition}[Lipschitz Continuity]
\label{prop:lipschitz}
Under regularity conditions, $s \mapsto \neff(s)$ and $s \mapsto \sigma_{\mathrm{post}}^2(s)$ are Lipschitz continuous.
\end{proposition}

\begin{proof}[Proof sketch]
The adaptive bandwidth $\bw(s)$ varies continuously (median of $k$-NN distances is Lipschitz). Kernel weights $w_i(s)$ are smooth compositions. By the chain rule, $\neff(s)$ inherits Lipschitz continuity.
\end{proof}

% ============================================================
\section{Discussion of Limitations}
\label{app:discussion}

\textbf{Spatial autocorrelation of residuals.} If residuals exhibit spatial autocorrelation (e.g., measured by Moran's $I$), correlated neighbors provide redundant information that inflates $\neff$. A design-effect correction can be applied:
\begin{equation}
  \neff^{\mathrm{adj}} = \frac{\neff}{1 + (k_{\mathrm{eff}} - 1) \cdot \bar{r}},
  \label{eq:neff-adj}
\end{equation}
where $\bar{r}$ is the average pairwise correlation among effective neighbors and $k_{\mathrm{eff}}$ is their count. This mirrors the classical design effect in survey sampling.

\textbf{Weight misspecification.} When the importance weights are misspecified ($w_i \not\propto dP_{\mathrm{test}}/dP_{\mathrm{calib}}$), the coverage guarantee of Theorem~\ref{thm:coverage} may not hold---this is precisely the same limitation as in weighted CP~\citep{tibshirani2019conformal}, who noted that the likelihood ratio must be known or accurately estimated. In practice,~\citet{tibshirani2019conformal} demonstrated that density ratios can be estimated via probabilistic classification (logistic regression or random forests), and such estimated weights still yield correct coverage. For WBCP, the $\neff$ diagnostic provides additional robustness information: low $\neff$ flags regions where few calibration points contribute, regardless of specification quality.

% ============================================================
\section{Full Ablation Study: 6-Method Comparison}
\label{app:ablation}

The main paper (Table~\ref{tab:cross_dataset}) reports only the proposed method (AdaGeoBCP) and three headline baselines (Standard CP, BQ-CP, GeoCP). This appendix adds the two ablation cells from the $2\times3$ grid in \S\ref{sec:geo}---fixed-bandwidth GeoBCP (which removes adaptive bandwidth from AdaGeoBCP) and frequentist AdaGeoCP (which removes the Bayesian overlay)---to verify that \emph{both} ingredients of AdaGeoBCP are necessary.

\subsection{Synthetic GRF benchmark, all 6 methods}

\begin{table}[h]
\centering
\caption{Synthetic spatial benchmark, all $6$ methods at $\ell=2.0$ (Moran's I $= 0.8677$, $\alpha = 0.1$). The two cells \textbf{not} reported in the main paper Table~\ref{tab:cross_dataset} are GeoBCP (fixed-bandwidth Bayesian) and AdaGeoCP (adaptive frequentist), shaded for clarity.}
\label{tab:synthetic}
\small
\begin{tabular}{lcccc}
\toprule
\textbf{Method} & \textbf{Coverage} & \textbf{Mean Width} & \textbf{Mean $\neff$} & \textbf{Mean $\sigma_{\mathrm{post}}$} \\
\midrule
Standard CP                                              & 0.9047          & 3.1849          & --- & --- \\
BQ-CP                                                    & 0.9187          & 3.4864          & 1500 & 0.0950 \\
GeoCP                                                    & 0.9647          & 1.9792          & --- & --- \\
\rowcolor{black!4} GeoBCP\,$\dagger$ (ablation)          & 0.9867          & 2.2224          & 62   & 0.1073 \\
\rowcolor{black!4} AdaGeoCP\,$\dagger$ (ablation)        & 0.9587          & \textbf{1.6642} & --- & --- \\
\textbf{AdaGeoBCP (proposed)}                            & 0.9827          & 1.8730          & 10   & 0.1081 \\
\bottomrule
\end{tabular}
\end{table}

\subsection{Real-world cross-dataset, all 6 methods}

\begin{table}[h]
\centering
\caption{Real-world cross-dataset summary, all $6$ methods. Bootstrap mean $\pm$ std across $13$ datasets ($10$ random splits each); right two columns report ``mean / loc-std'' (across-locations mean and across-locations std). Ablation cells (GeoBCP, AdaGeoCP) shaded. Both removed ingredients drop coverage below target on most datasets, confirming that adaptive bandwidth and Bayesian overlay are jointly required.}
\label{tab:cross_dataset_full}
\small
\setlength{\tabcolsep}{3pt}
\begin{tabular}{lccc}
\toprule
\textbf{Method} & \textbf{Coverage} & $\boldsymbol{\neff}$ \textbf{(mean / loc-std)} & $\boldsymbol{\sigma_{\mathrm{post}}}$ \textbf{(mean / loc-std)} \\
\midrule
Standard CP                                                          & $0.903 \pm 0.009$           & ---                       & ---                       \\
BQ-CP                                                                & $0.928 \pm 0.011$           & $242.9$ / $0.0$           & $0.58$ / $0.02$           \\
GeoCP                                                                & $0.824 \pm 0.034$           & ---                       & ---                       \\
\rowcolor{black!4} GeoBCP\,$\dagger$ (ablation)                      & $0.876 \pm 0.025$           & $23.2$ / $6.5$            & $1.17$ / $1.10$           \\
\rowcolor{black!4} AdaGeoCP\,$\dagger$ (ablation)                    & $0.895 \pm 0.007$           & ---                       & ---                       \\
\textbf{AdaGeoBCP (proposed)}                                        & $\mathbf{0.925 \pm 0.009}$  & $171.3$ / $\mathbf{13.1}$ & $1.19$ / $\mathbf{0.95}$  \\
\bottomrule
\end{tabular}
\end{table}

\subsection{Ablation read-out}

\textbf{Removing the Bayesian overlay} (AdaGeoBCP $\to$ AdaGeoCP) drops cross-dataset target-coverage attainment from $13/13$ to $3/13$ (mean coverage $0.925 \to 0.895$); the proposed method's posterior is required to recover from the under-coverage of bandwidth-localized weighted quantiles in the small-$\neff$ regime.

\textbf{Removing adaptive bandwidth} (AdaGeoBCP $\to$ GeoBCP) drops attainment to $1/13$ (mean coverage $0.925 \to 0.876$): the auto-bandwidth used by GeoCP/GeoBCP gives $\neff \approx 7$ on most non-Seattle datasets, leaving the underlying weighted quantile too biased even with the Bayesian posterior on top.

\textbf{Removing both} (Standard CP and BQ-CP rows) recovers the i.i.d.\ baselines: BQ-CP attains $13/13$ but only because uniform weights ignore distribution shift entirely (cf.\ the $\neff$ loc-std $0.0$ column---no spatial adaptation).

The two ingredients are therefore complementary: the Bayesian overlay alone (GeoBCP) does not fix the under-coverage induced by fixed bandwidth, and adaptive bandwidth alone (AdaGeoCP) does not address the meta-uncertainty gap. Only the combined AdaGeoBCP simultaneously handles distribution shift, attains target coverage, and produces spatially varying $\sigma_{\mathrm{post}}$ / $\neff$ diagnostics.

% ============================================================
\section{Coverage--Width Frontier under \texorpdfstring{$\beta$}{β} Sweep}
\label{app:beta_sweep}

The headline tables in \S\ref{sec:experiments} use a single fixed HPD confidence $\beta = 0.9$, which deliberately trades coverage margin for posterior protection and is therefore not directly comparable to weighted CP at \emph{matched} coverage. We address this by sweeping $\beta \in \{0.50, 0.60, 0.70, 0.80, 0.90, 0.95, 0.99\}$ on the Seattle dataset and recomputing the HPD threshold (and thus coverage and width) from the \emph{same} posterior samples produced by Algorithm~\ref{alg:wbqcp}; no additional model fitting is required.

\textbf{Frontier behaviour.} Figure~\ref{fig:beta_frontier} plots the resulting (mean width, coverage) frontier for the three Bayesian variants, with the corresponding weighted-CP point-estimate baselines (CP, GeoCP, AdaGeoCP) overlaid as filled markers.

\begin{figure}[ht]
  \centering
  \includegraphics[width=0.85\linewidth]{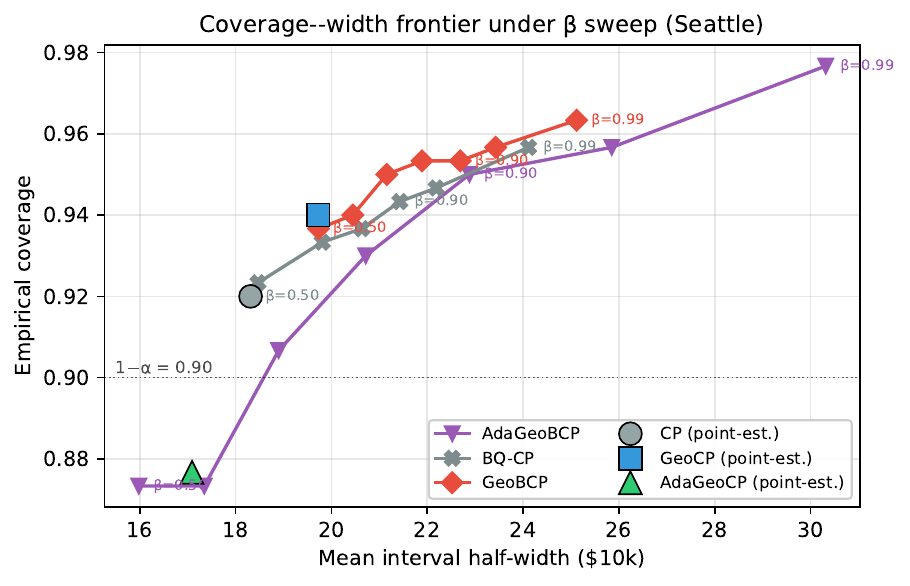}
  \caption{Coverage--width frontier obtained by sweeping the HPD confidence $\beta$ on the Seattle dataset. Each Bayesian method (BQ-CP, GeoBCP, AdaGeoBCP) traces a curve from $\beta = 0.50$ (left, narrowest, lowest coverage) to $\beta = 0.99$ (right, widest, highest coverage). At $\beta = 0.5$, BQ-CP and GeoBCP lie close to their corresponding point-estimate baselines (filled markers), while AdaGeoBCP remains noticeably different in the small-$\neff$ regime. The dashed horizontal line marks the marginal target $1-\alpha = 0.9$.}
  \label{fig:beta_frontier}
\end{figure}

\textbf{Behaviour at $\beta = 0.5$.} Table~\ref{tab:beta_sweep} shows that at $\beta = 0.5$, BQ-CP and GeoBCP are numerically close to their corresponding point-estimate baselines in both coverage and width, whereas AdaGeoBCP remains visibly different from AdaGeoCP because the posterior is much more diffuse in the small-$\neff$ regime. We therefore interpret $\beta$ not as yielding exact recovery of weighted CP, but as providing a practical operating-point control on a common posterior.

\begin{table}[ht]
\centering
\caption{Seattle: coverage and mean interval half-width as a function of $\beta$, recomputed from the same posterior samples. Reference point-estimate baselines: Standard CP (cov $0.9200$, width $18.31$), GeoCP (cov $0.9400$, width $19.72$), AdaGeoCP (cov $0.8767$, width $17.09$).}
\label{tab:beta_sweep}
\small
\begin{tabular}{l|cccccccc}
\toprule
\textbf{Method} & \textbf{Metric} & $\beta=0.50$ & $0.60$ & $0.70$ & $0.80$ & $0.90$ & $0.95$ & $0.99$ \\
\midrule
\multirow{2}{*}{BQ-CP}     & Coverage   & 0.9200 & 0.9233 & 0.9333 & 0.9367 & 0.9433 & 0.9467 & 0.9567 \\
                            & Mean width & 18.32 & 18.47 & 19.81 & 20.63 & 21.43 & 22.19 & 24.12 \\
\midrule
\multirow{2}{*}{GeoBCP}    & Coverage   & 0.9367 & 0.9400 & 0.9500 & 0.9533 & 0.9533 & 0.9567 & 0.9633 \\
                            & Mean width & 19.73 & 20.45 & 21.16 & 21.89 & 22.69 & 23.43 & 25.12 \\
\midrule
\multirow{2}{*}{AdaGeoBCP} & Coverage   & 0.8733 & 0.8733 & 0.9067 & 0.9300 & 0.9500 & 0.9567 & 0.9767 \\
                            & Mean width & 15.98 & 17.35 & 18.90 & 20.72 & 22.88 & 25.85 & 30.32 \\
\bottomrule
\end{tabular}
\end{table}

\textbf{Reading the frontier.} Three observations are worth emphasizing. \emph{(i)}~Within each row, $\sigma_{\mathrm{post}}$ and $\neff$ are unchanged across $\beta$ (they are properties of the posterior, not of the threshold quantile), so the diagnostic value of WBCP is decoupled from the conservativeness choice. \emph{(ii)}~At $\beta = 0.5$ the AdaGeoBCP frontier passes \emph{below} the marginal target ($\text{cov}=0.873$), reproducing the AdaGeoCP under-coverage failure mode of Table~\ref{tab:seattle_50split}; this is expected because the posterior median tracks the (under-covering) point estimate. Moving to $\beta=0.9$ recovers $0.95$ coverage with width comparable to BQ-CP. \emph{(iii)}~The frontiers for BQ-CP and GeoBCP rise gently with $\beta$, while AdaGeoBCP's frontier is much steeper because its $\neff$ is small ($\approx 13$): the Dirichlet posterior is diffuse and the gap between successive HPD quantiles is large. This empirical pattern is in line with the spatial adaptivity suggested by Theorem~\ref{thm:concentration}.

We therefore recommend that practitioners report the $(\beta, \text{coverage}, \text{width})$ triple rather than a single fixed-$\beta$ point, with the choice of $\beta$ guided by the application's tolerance for over- vs.\ under-coverage; $\sigma_{\mathrm{post}}$ and $\neff$ remain available as diagnostics at every operating point.

% ============================================================
\section{Simulation Study Results}
\label{sec:simulation_study}

The numerical synthetic results for all $6$ methods at $\ell=2.0$ are reported as part of the ablation study, Table~\ref{tab:synthetic} in Appendix~\ref{app:ablation}. The figures below show the simulation design and the spatial-autocorrelation sweep across $\ell\in\{0.5,1,2,4,8,16\}$.

\begin{figure}[ht]
    \centering
    \includegraphics[width=\linewidth]{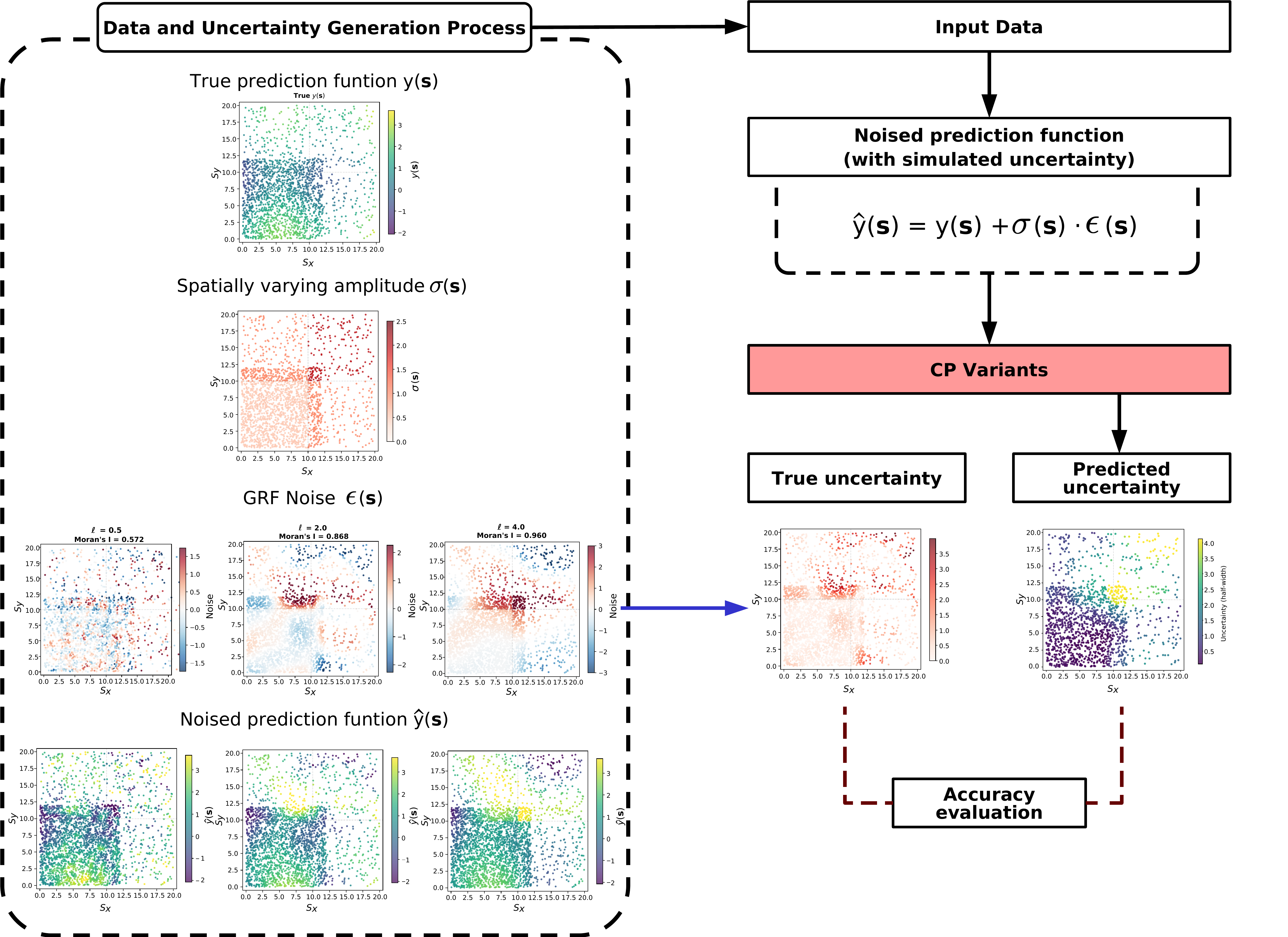}
    \caption{An overview of simulation design. A true prediction function $y(\textbf{s})$ and a noised prediction function with simulated uncertainty $\hat{y}(s)=y(\textbf{s})+\sigma(\textbf{s})\cdot\epsilon(\textbf{s})$ are artificially formulated. Different CP variants are employed to generate predicted uncertainty, which is compared with the designed true uncertainty.}
    \label{fig:geobcp_simulation_design}
\end{figure}

\begin{figure}[ht]
    \centering
    \includegraphics[width=\linewidth]{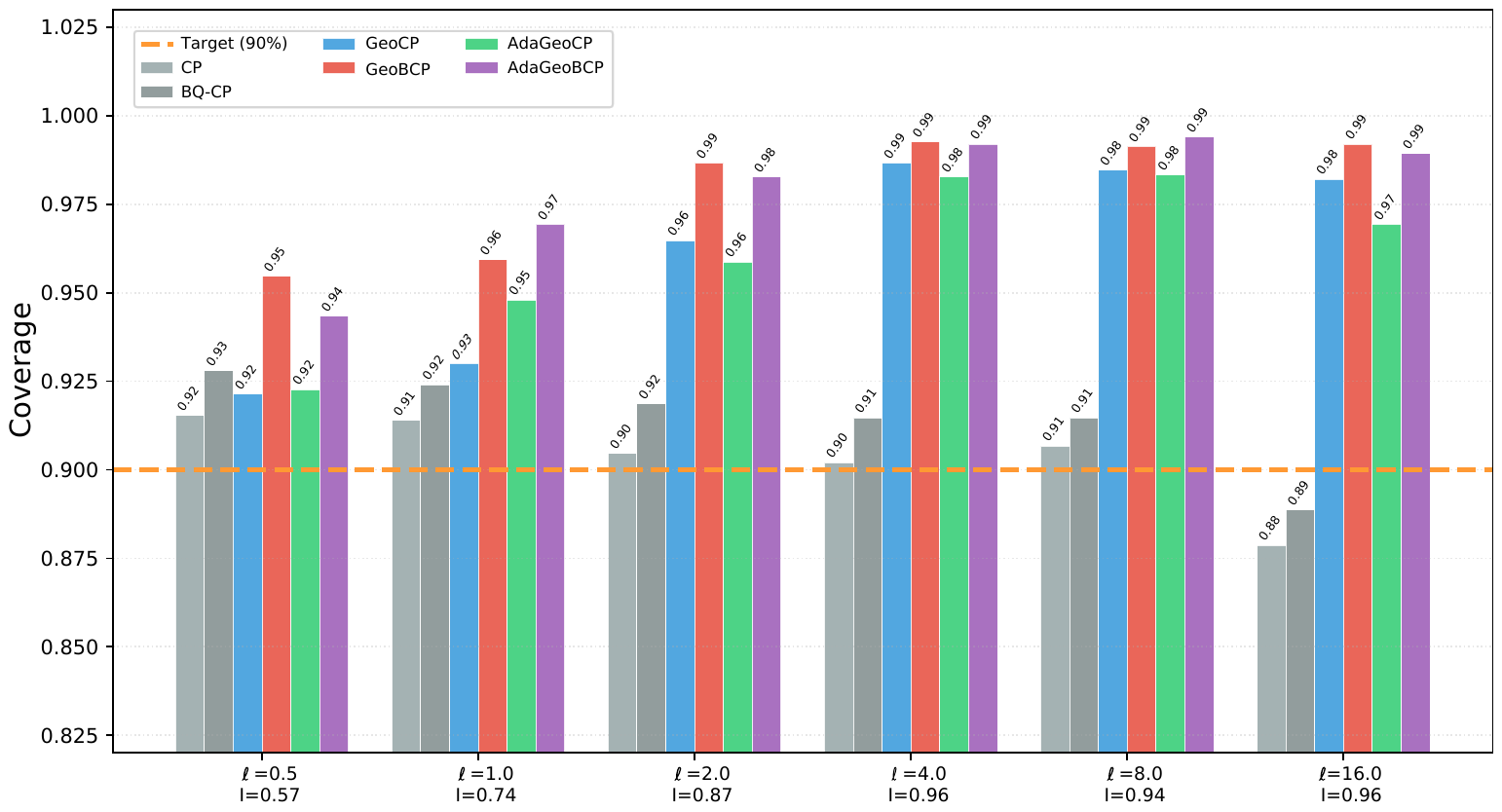}
    \caption{Coverage across all five variants at different spatial autocorrelation levels.}
    \label{fig:geobcp_simulation_coverage}
\end{figure}

\begin{figure}[ht]
    \centering
    \includegraphics[width=\linewidth]{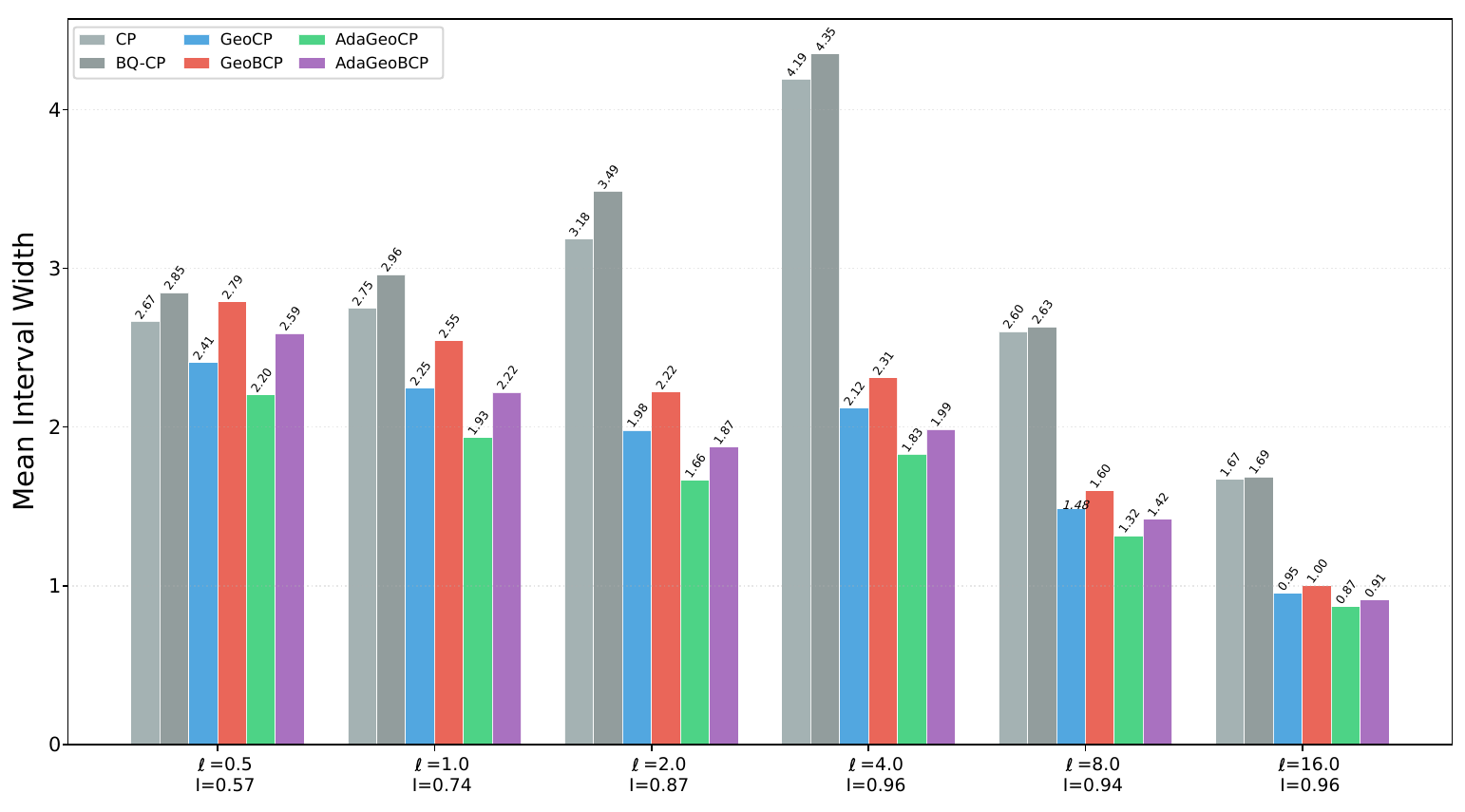}
    \caption{Prediction interval width across all five variants at different spatial autocorrelation levels.}
    \label{fig:geobcp_simulation_interval_width}
\end{figure}

\section{Real-world Study Results}\label{sec:real_world_study}

This appendix provides the 50-split Seattle headline bootstrap, dataset-detail tables, and complementary cross-dataset distributional views (Figure~\ref{fig:geobcp_map_posteriors} appears in the main paper).

\subsection{Dataset Details and Prediction Models}
\label{app:dataset_details}

Table~\ref{tab:datasets} reports the dataset, application domain, raw size $n$, number of features $p$, prediction target, target value range, and source for all $13$ real-world datasets. All targets are continuous and the task is regression. We use a standard XGBoost regressor~\citep{chen2016xgboost} as the base $\hat{f}$ for every dataset, with shared hyperparameters $n_{\mathrm{estimators}}=500$, $\max\_\mathrm{depth}=3$, learning rate $0.1$, $\mathrm{min\_child\_weight}=1$, and $\mathrm{colsample\_bytree}=1.0$; per-dataset XGBoost test-set $R^2$ is reported in Table~\ref{tab:datasets} (computed once at $\mathrm{seed}=42$ to summarize the base learner's quality). For datasets with $n>5000$ we subsample $5000$ rows uniformly at random for tractable bootstrap repetition; for FIA Forest we subsample $3000$ to match the headline Seattle scale. The 80/10/10 train / calibration / test split is then resampled across $10$ random seeds to produce bootstrap CIs in Table~\ref{tab:cross_dataset}. The conformity score is the absolute residual $\rho_i = |\hat{f}(X_i) - Y_i|$ for all datasets. Coordinate columns are EPSG:4326 (lat/lon) except Seattle (UTM~10N reprojected to EPSG:4326 for spatial weighting) and Voting (Albers Equal Area projection in metres). Categorical features are label-encoded; missing feature values are filled with $0$.

\begin{table}[ht]
\centering
\caption{Real-world datasets used in the cross-dataset evaluation (\S\ref{sec:experiments} and Table~\ref{tab:cross_dataset}). $n$ is the number of records before any subsampling; for the bootstrap experiment, datasets with $n>5000$ are uniformly subsampled. Test-set $R^2$ is the XGBoost base-model fit at random seed $42$. Source citations and access URLs are listed below the table.}
\label{tab:datasets}
\small
\setlength{\tabcolsep}{3.5pt}
\begin{tabular}{lllrrlcl}
\toprule
\textbf{Dataset} & \textbf{Domain} & \textbf{Source} & $\boldsymbol{n}$ & $\boldsymbol{p}$ & \textbf{Target $Y$} & \textbf{Range} & \textbf{XGB $R^2$} \\
\midrule
Seattle             & Housing       & Kaggle KC$^{\,a}$               & 3,000   & 10 & price (\$10k)              & $[8.2,\,363.5]$  & $0.87$ \\
Climate             & Climate       & ERA5$^{\,b}$                    & 83,731  &  9 & annual climate metric      & $[0.17,\,30.5]$  & $0.92$ \\
Hydro               & Hydrology     & CAMELS$^{\,c}$                  & 671     & 10 & runoff/streamflow metric   & $[0.0,\,9.7]$    & $0.92$ \\
Voting              & Politics      & MIT Election Lab$^{\,d}$        & 3,108   & 14 & vote-share metric (\%)     & $[3.1,\,94.5]$   & $0.82$ \\
Forest              & Forestry      & USDA FIA$^{\,e}$                & 79,878  & 12 & forest stand metric        & $[0.0,\,767.4]$  & $0.65$ \\
Health\_ARTHRITIS   & Public health & CDC PLACES$^{\,f}$              & 2,350   &  8 & arthritis prevalence (\%)  & $[16.2,\,36.6]$  & $0.86$ \\
Health\_BPHIGH      & Public health & CDC PLACES$^{\,f}$              & 2,350   &  8 & high-BP prevalence (\%)    & $[21.6,\,52.7]$  & $0.91$ \\
Health\_CANCER      & Public health & CDC PLACES$^{\,f}$              & 2,350   &  8 & cancer prevalence (\%)     & $[4.6,\,8.2]$    & $0.83$ \\
Health\_CASTHMA     & Public health & CDC PLACES$^{\,f}$              & 2,350   &  8 & current-asthma prev.\ (\%) & $[7.3,\,15.2]$   & $0.80$ \\
Health\_DEPRESSION  & Public health & CDC PLACES$^{\,f}$              & 2,350   &  8 & depression prevalence (\%) & $[12.5,\,35.7]$  & $0.84$ \\
Health\_DIABETES    & Public health & CDC PLACES$^{\,f}$              & 2,350   &  8 & diabetes prevalence (\%)   & $[6.2,\,21.5]$   & $0.93$ \\
Health\_OBESITY     & Public health & CDC PLACES$^{\,f}$              & 2,350   &  8 & obesity prevalence (\%)    & $[17.7,\,53.0]$  & $0.91$ \\
Health\_STROKE      & Public health & CDC PLACES$^{\,f}$              & 2,350   &  8 & stroke prevalence (\%)     & $[1.9,\,7.6]$    & $0.91$ \\
\bottomrule
\end{tabular}

\vspace{4pt}
{\footnotesize
\begin{tabular}{@{}p{0.97\linewidth}@{}}
$^{a}$~``House Sales in King County, USA'' Kaggle dataset~\citep{kc_housing_kaggle}; raw data: \url{https://www.kaggle.com/datasets/harlfoxem/housesalesprediction}. \\
$^{b}$~ERA5 global reanalysis~\citep{hersbach2020era5}, accessed via the Copernicus Climate Data Store (\url{https://cds.climate.copernicus.eu}); we extract a US-grid subset and aggregate to an annual climate variable. \\
$^{c}$~CAMELS catchment attributes and meteorology benchmark~\citep{addor2017camels}; data: \url{https://gdex.ucar.edu/dataset/camels.html}. \\
$^{d}$~U.S.\ County Presidential Election Returns 2000--2020, MIT Election Data and Science Lab~\citep{mit_election_lab_county} via Harvard Dataverse: \url{https://doi.org/10.7910/DVN/VOQCHQ}. \\
$^{e}$~USDA Forest Service Forest Inventory and Analysis (FIA) Database~\citep{usda_fia}: \url{https://www.fs.usda.gov/research/programs/fia}; we use the public plot-level export. \\
$^{f}$~U.S.\ CDC PLACES Local Data for Better Health, county-level prevalence release~\citep{cdc_places}: \url{https://www.cdc.gov/places}; each Health\_* dataset uses the indicator named after the suffix (e.g.\ Health\_DIABETES uses the diabetes prevalence indicator) joined with $8$ standard demographic / socioeconomic county-level covariates as features. \\
\end{tabular}}
\end{table}

\begin{table}[ht]
\centering
\caption{Seattle headline bootstrap ($\alpha=0.1$, $50$ random splits). Each cell reports mean $\pm$ standard deviation over splits. Proposed methods in \textbf{bold}; \textbf{bold} also marks the highest Coverage and the smallest Mean Width. $\dagger$~AdaGeoCP fails the target $1-\alpha=0.9$ on this dataset (mean $0.877$, consistent across all $50$ splits).}
\label{tab:seattle_50split}
\small
\setlength{\tabcolsep}{4pt}
\begin{tabular}{lcccc}
\toprule
\textbf{Method} & \textbf{Coverage} & \textbf{Mean Width (\$10k)} & \textbf{Mean $\neff$} & \textbf{Mean $\sigma_{\mathrm{post}}$} \\
\midrule
Standard CP            & $0.909 \pm 0.023$            & $\mathbf{17.68 \pm 1.56}$ & --- & --- \\
BQ-CP                  & $0.929 \pm 0.021$            & $20.41 \pm 1.98$         & $300 \pm 0$    & $1.90 \pm 0.45$ \\
GeoCP                  & $0.921 \pm 0.022$            & $18.66 \pm 1.61$         & --- & --- \\
\textbf{GeoBCP}        & $\mathbf{0.941 \pm 0.021}$   & $21.98 \pm 2.06$         & $213 \pm 3$    & $2.51 \pm 0.55$ \\
AdaGeoCP$^{\dagger}$   & $0.877 \pm 0.026^{\dagger}$  & $18.21 \pm 1.62$         & --- & --- \\
\textbf{AdaGeoBCP}     & $0.929 \pm 0.020$            & $24.91 \pm 2.85$         & $13.5 \pm 0.5$ & $6.80 \pm 1.32$ \\
\bottomrule
\end{tabular}
\end{table}

\begin{figure}[ht]
    \centering
    \includegraphics[width=\linewidth]{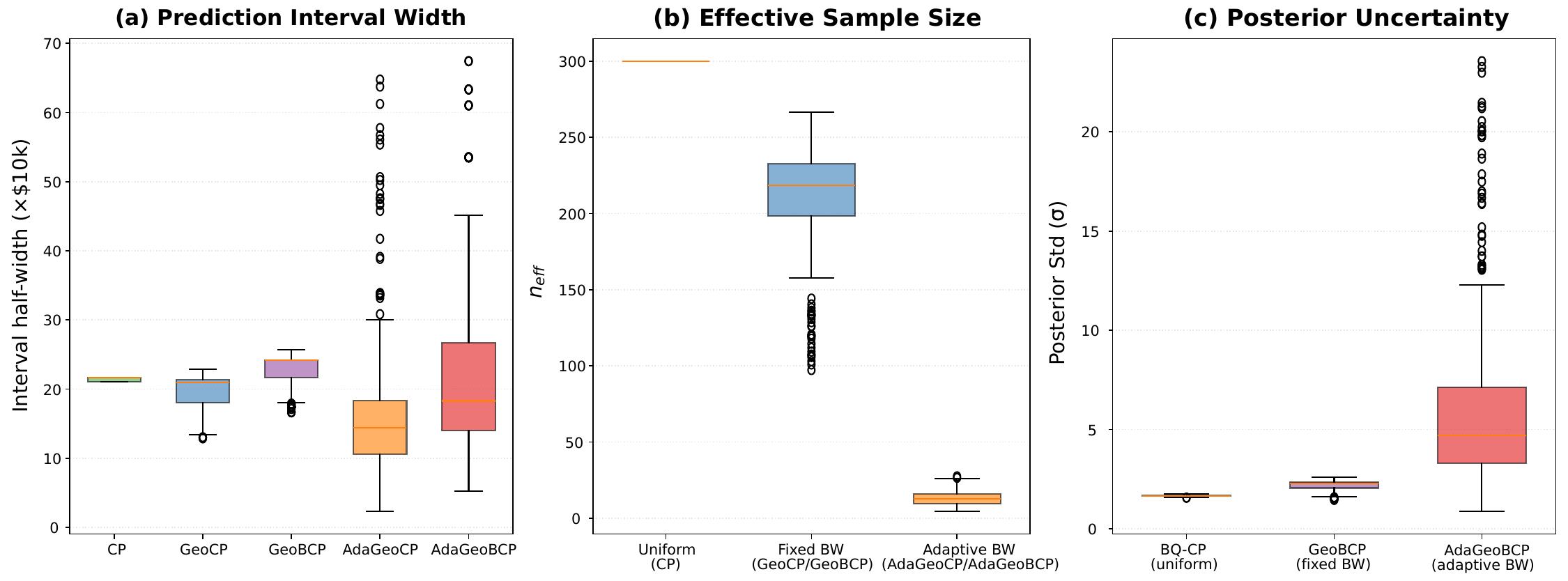}
    \caption{Distribution comparison across all five variants (Seattle house price as an example). (a) Prediction interval half-width. (b) Effective sample size. (c) Posterior standard deviation.}
    \label{fig:variants_boxplot}
\end{figure}

\begin{figure}[ht]
    \centering
    \includegraphics[width=\linewidth]{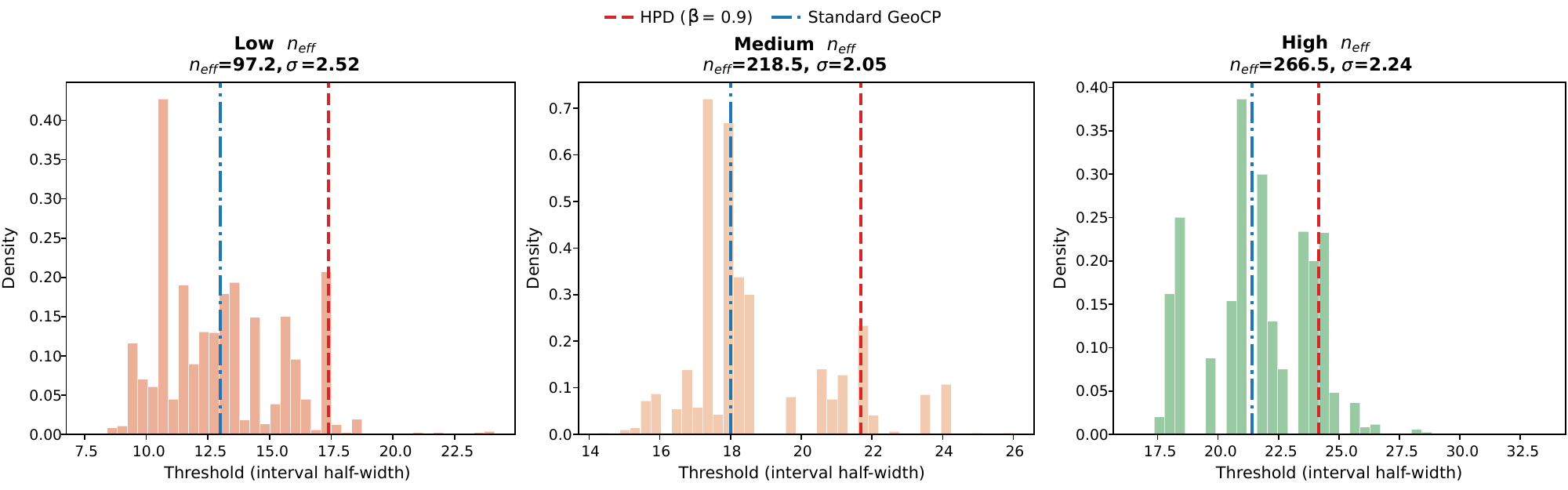}
    \caption{Posterior distribution of prediction interval width (Seattle house price as an example). The samples are selected by $\neff$. Left: low $\neff$, middle: medium $\neff$, right: high $\neff$. The higher the $\neff$ is, the narrower its posterior distribution is.}
    \label{fig:posterior_distribution_samples}
\end{figure}

\begin{figure}[ht]
    \centering
    \includegraphics[width=\linewidth]{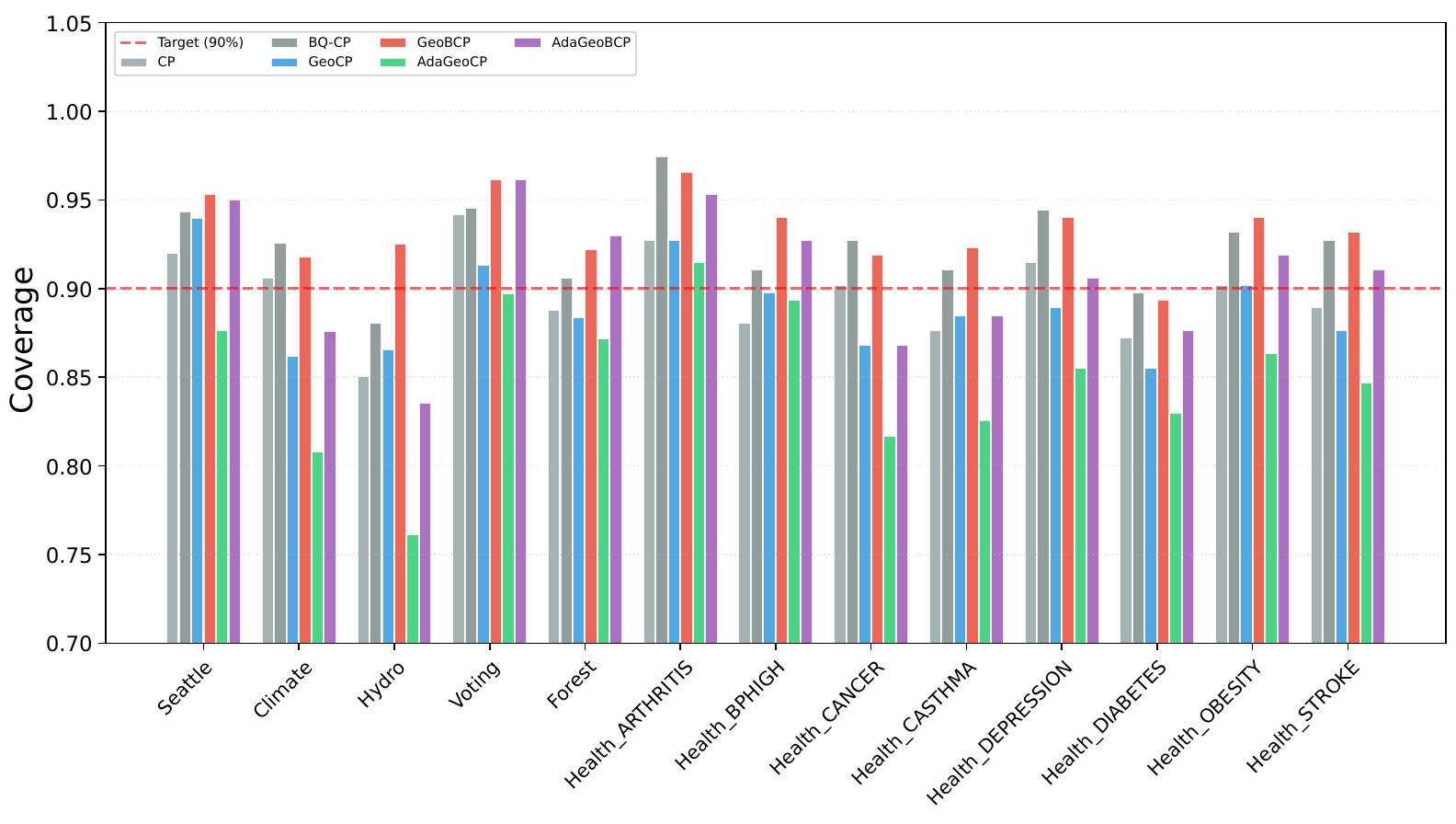}
    \caption{Empirical coverage of all five variants across 13 geospatial datasets (single-split per dataset).}
    \label{fig:coverage_comparison_all_datasets}
\end{figure}

\begin{figure}[ht]
    \centering
    \includegraphics[width=\linewidth]{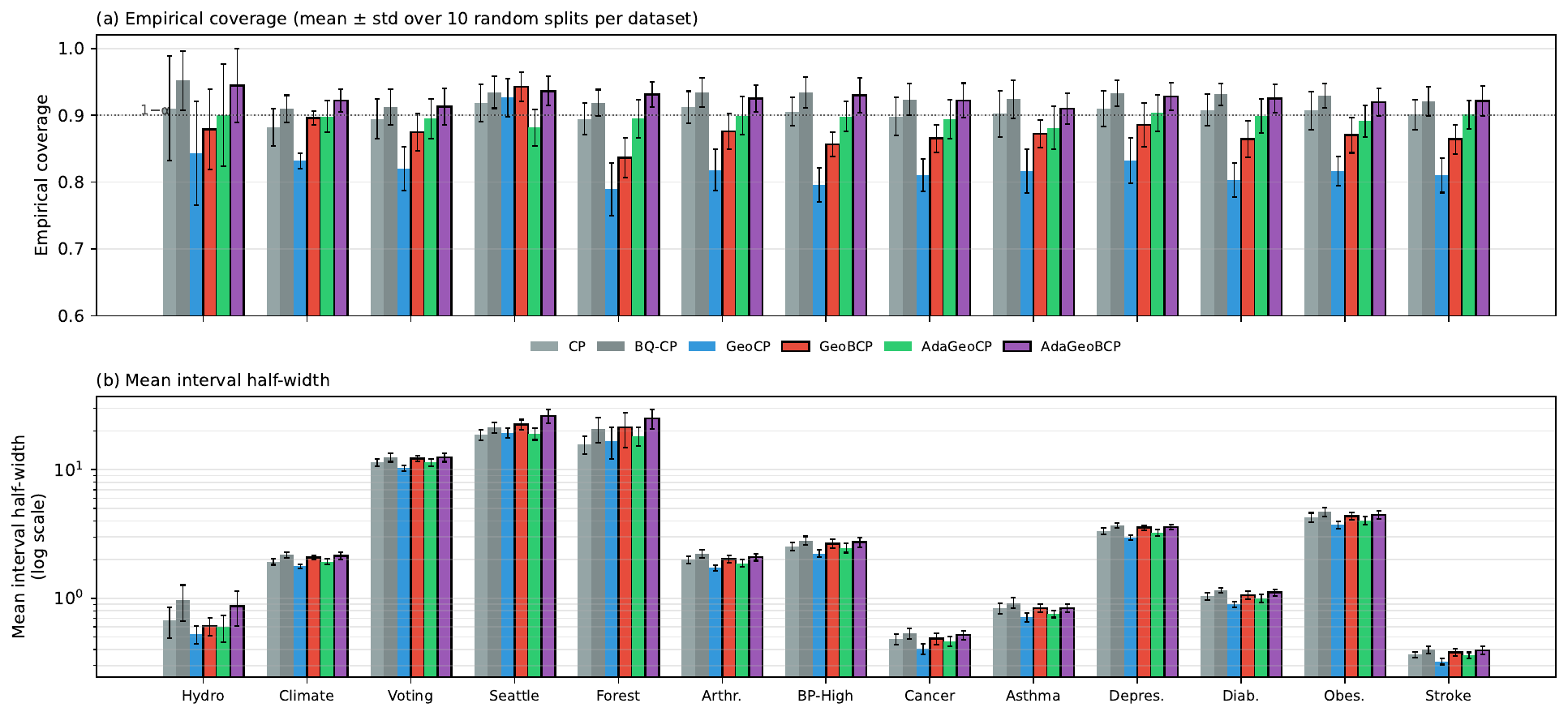}
    \caption{Bootstrap evaluation across all $13$ geospatial datasets ($10$ random train/calibration/test splits per dataset, $\alpha=0.1$, $\beta=0.9$); bars show the bootstrap mean and error bars show the bootstrap standard deviation. Proposed methods (GeoBCP, AdaGeoBCP) are outlined in black. (a)~Empirical coverage with the marginal target $1-\alpha = 0.9$ marked. AdaGeoBCP is the only method that handles distribution shift \emph{and} maintains coverage at or above target on all $13/13$ datasets (matching the i.i.d.\ baseline BQ-CP at $13/13$, but additionally handling shift). Fixed-bandwidth GeoCP and GeoBCP under-cover on most datasets ($1/13$ each), confirming that locality-induced under-coverage in weighted CP is a real and widespread phenomenon, and that the Bayesian overlay alone is not enough---adaptive bandwidth and the Bayesian posterior together are required, as instantiated by AdaGeoBCP. (b)~Mean interval half-width on a log scale (data-set-specific scales differ by two orders of magnitude). Bootstrap raw data are in \texttt{bootstrap\_all\_datasets.csv} in the supplementary material.}
    \label{fig:bootstrap_matrix}
\end{figure}

\end{document}